%% file: draft.tex
\documentclass{article} % For LaTeX2e
\usepackage{iclr2020_conference,times}

\input{math_commands.tex}

\usepackage{subcaption}
\usepackage{hyperref}
\usepackage{url}
\usepackage{graphicx}
\usepackage{tabularx}
\usepackage{amsthm}
\usepackage{booktabs}
\usepackage[ruled,vlined]{algorithm2e}
\usepackage{listings}
\usepackage{wrapfig}
\definecolor{mygreen}{rgb}{0,0.6,0}
\newsavebox{\codebox}

\graphicspath{ {./images/} }

\definecolor{Green}{RGB}{39,174,96}
\definecolor{Red}{RGB}{200,50,50}
\definecolor{Gray}{RGB}{105,105,105}

\newtheorem{thm}{Theorem}[section]

\title{Conditional Negative Sampling for Contrastive Learning of Visual Representations}

\author{
  Mike Wu$^1$, Milan Mosse$^{1,3}$, Chengxu Zhuang$^2$, Daniel Yamins$^{1,2}$, Noah Goodman$^{1,2}$ \\
  Department of Computer Science$^1$, Psychology$^2$, and Philosophy$^3$ \\
  Stanford University \\
  \texttt{\{wumike, chengxuz, mmosse19, yamins, ngoodman\}@stanford.edu}
}

\iclrfinalcopy % Uncomment for camera-ready version, but NOT for submission.
\begin{document}

\maketitle

\begin{abstract}
Recent methods for learning unsupervised visual representations, dubbed contrastive learning, optimize the noise-contrastive estimation (NCE) bound on mutual information between two views of an image. NCE uses randomly sampled negative examples to normalize the objective. In this paper, we show that choosing difficult negatives, or those more similar to the current instance, can yield stronger representations. To do this, we introduce a family of mutual information estimators that sample negatives conditionally -- in a ``ring'' around each positive. We prove that these estimators lower-bound mutual information, with higher bias but lower variance than NCE. Experimentally, we find our approach, applied on top of existing models (IR, CMC, and MoCo) improves accuracy by 2-5\% points in each case, measured by linear evaluation on four standard image datasets. Moreover, we find continued benefits when transferring features to a variety of new image distributions from the Meta-Dataset collection and to a variety of downstream tasks such as object detection, instance segmentation, and keypoint detection.
%While much of the attention in contrastive learning has focused on augmentations and architecture, we showcase negative sampling as a less explored but important direction for future research.
\end{abstract}

% \ndg{get rid of the verb "mining" except in related work where they use it...}
\section{Introduction}
Supervised learning algorithms have given rise to human-level performance in several visual tasks \citep{russakovsky2015imagenet,redmon2016you,he2017mask}, relying heavily on large image datasets paired with  semantic annotations. These annotations vary in difficulty and cost, spanning from simple class labels \citep{deng2009imagenet} to more granular descriptions like bounding boxes \citep{everingham2010pascal,lin2014microsoft} and key points \citep{lin2014microsoft}. As it is impractical to scale high quality annotations to the size that modern deep learning demands, this reliance on supervision poses a barrier to widespread adoption.
In response, we have seen the growth of \textit{un}-supervised approaches to learning representations, or embeddings, that are general and can be re-used for many tasks at once. In natural language processing, this approach has been highly successful, resulting in the popular GPT \citep{radford2018improving,radford2019language,brown2020language} and BERT \citep{devlin2018bert,liu2019roberta} models. While supervised pretraining is still dominant in computer vision, recent approaches using ``contrastive'' objectives, have sparked great interest from the research community \citep{wu2018unsupervised,oord2018representation,hjelm2018learning,zhuang2019local,henaff2019data,misra2020self,he2019momentum,chen2020simple,chen2020improved,grill2020bootstrap}.
In the last two years, contrastive methods have already achieved remarkable results, quickly closing the gap to supervised methods \citep{he2016deep,he2019momentum,chen2020simple,grill2020bootstrap}.

Recent contrastive algorithms were developed as estimators of mutual information \citep{oord2018representation,hjelm2018learning,bachman2019learning}, building on the intuition that a good low-dimensional ``representation'' would be one that linearizes the useful information embedded within a high-dimensional data point. In the visual domain, these estimators optimize an encoder by maximizing the
similarity of encodings for two augmentations (i.e.~transformations) of the same image.
Doing so is trivial unless this similarity function is normalized. This is done by using ``negative examples'', contrasting an image (e.g. of a cat) with a set of possible other images (e.g. of dogs, tables, cars, etc.). We hypothesize that the manner in which we choose these negatives greatly effects the quality of the representations. For instance, differentiating a cat from other breeds of cats is visually more difficult than differentiating a cat from other classes. The encoder may thus have to focus on more granular, semantic information (e.g. fur patterns) that may be useful for downstream visual tasks (e.g.~object classification). While research in contrastive learning has explored architectures, augmentations, and pretext tasks, there has been little attention given to how one chooses negative samples beyond the common tactic of uniformly sampling from the training dataset.
% popular approaches choose negatives uniformly from the training dataset.
%We hypothesize that a more careful negative mining protocol may be a powerful approach to learn even better representations.
% \ndg{key points to hit: harder negative examples seem like a good idea, but it's not clear that the result still lower bounds MI, if it does it's not clear that it will work well, we show that it is a lower bound and it's looser but lower variance than the standard procedure. we then show that a particular version, ring, can be simply added to several popular nce algorithms improving the performance in (ever? most?) case.}

While choosing more difficult negatives seems promising, there are several unanswered theoretical and practical questions.
Naively choosing difficult negatives may yield an objective that no longer bounds mutual information.
Since such bounds are the basis for many contrastive objectives, and have been used for choosing good augmentations \citep{tian2020makes} and other innovations, it is desirable to use harder negatives without losing this property.
%Primarily, the relation of the negative sampling distribution and information theory is unclear. Naively choosing difficult negatives may no longer bound mutual information, which has been a useful theoretical framework for choosing good augmentations \citep{tian2020makes}, and the basis for many contrastive objectives.
Moreover, even if choosing difficult negatives is theoretically justified, we do not know if it will yield representations better for downstream tasks.
In this paper, we present a new family of estimators that supports sampling negatives from a particular class of conditional distributions.
% \ndg{should we call these variational (and the method VCE)? we don't vary anything... maybe conditional since they're conditioned on the positive example? CNCE (conditional-noise contrastive estimation) instead of VCE?}
We then prove that this family remains a lower bound of mutual information.
Moreover, we show that while they are a looser bound than the well-known noise contrastive estimator, estimators in this family have lower variance.
We propose a particular method, the Ring model, within this family for choosing negatives that are close, but not too close, to the positive example.
We then apply Ring to representation learning, where it is straightforward to adjust state-of-the-art contrastive objectives (e.g. MoCo, CMC) to sample harder negatives.
We find that Ring negatives improve transfer performance across datasets and across underlying objectives, making this an easy and useful addition to contrastive learning methods.
%Of the many knobs that one can turn to explore new contrastive losses, we hope our promising results show that negative sampling is one to consider for future work.

\section{Background}
% 1. explain contrastive learning
% 2. explain why many negative samples is hard
% 3. explain IR, MoCo, SimCLR's approach
% 4. connect these to NCE and thus to mutual information, which serves as an important motivation for the field (cite, cite, cite)
Recent contrastive learning has focused heavily on exemplar-based objectives, where examples, or instances, are compared to one another to learn a representation.
Many of these exemplar-based losses \citep{hjelm2018learning,wu2018unsupervised,bachman2019learning,zhuang2019local,tian2019contrastive,chen2020simple} are equivalent to noise contrastive estimation, or NCE \citep{gutmann2010noise,oord2018representation,poole2019variational}, which is a popular lower bound on the mutual information, denoted by $\mathcal{I}$, between two random variables. This connection is well-known and stated in several works \citep{chen2020simple,tschannen2019mutual,tian2020makes}, as well as explicitly motivating several algorithms (e.g. Deep InfoMax \citep{hjelm2018learning,bachman2019learning}), and choices of image views \citep{tian2020makes}. To review, recall the NCE objective:
\begin{equation}
    \mathcal{I}(U; V) \geq \mathcal{L}_{\text{NCE}}(u_i,v_i) = \mathbf{E}_{u_i,v_i \sim p(u,v)}\mathbf{E}_{v_{1:k} \sim p(v)} \left[ \log \frac{e^{f_\theta(u_i,v_i)}}{\frac{1}{k+1}\sum_{j\in\{i,1:k\}} e^{f_\theta(u_i,v_j)}} \right]
\label{eq:nce}
\end{equation}
where $u,v$ are realizations of two random variables of interest, $U$ and $V$. We call $v_{1:k} = \{v_1, \ldots v_k \}$ ``negative examples''
that normalize the numerator with respect to other possible realizations of $V$. A proof of the inequality in Eq.~\ref{eq:nce} can be found in \cite{poole2019variational}.

Now, suppose $U$ and $V$ are derived from the same random variable $X$, and we are given a dataset $\mathcal{D} = \{ x_i \}_{i=1}^n$ of $n$ values that $X$ can take, sampled from a distribution $p(x)$. Define $\mathcal{T}$ as a family of functions where each member $t: X \rightarrow X$ maps one realization of $X$ to another. We call a transformed input $t(x)$ a ``view'' of $x$.
In vision, $t \in \mathcal{T}$ is user-specified to be a composition of cropping, adding color jitter, gaussian blurring, among many others \citep{wu2018unsupervised,bachman2019learning,chen2020simple}. The choice of view family is a primary determinant of how successful a contrastive algorithm is \citep{tian2020makes,tian2019contrastive,chen2020simple}. Finally, let $p(t)$ denote a distribution over $\mathcal{T}$ from which we can sample, a common choice being uniform over $\mathcal{T}$.

Next, introduce an encoder $g_\theta: X \rightarrow \mathbf{R}^d$ that maps an instance to a representation. Then, a general  contrastive objective for the $i$-th example in $\mathcal{D}$ is:
\begin{equation}
    \mathcal{L}(x_i) = \mathbf{E}_{t,t',t_{1:k} \sim p(t)} \mathbf{E}_{x_{1:k} \sim p(x)}\left[ \log \frac{e^{g_\theta(t(x_i))^T g_\theta(t'(x_i)) / \tau}}{\frac{1}{k+1} \sum_{j \in \{i,1:k\}} e^{g_\theta(t(x_i))^T g_\theta(t_j(x_j)) / \tau}} \right]
\label{eq:contrastive}
\end{equation}
where $\tau$ is a temperature used to scale the dot product. The equivalence to NCE is immediate given $f_\theta(u,v) = g_\theta(u)^T g_\theta(v) / \tau$.
We can interpret maximizing Eq.~\ref{eq:contrastive} as
choosing an embedding that pulls two views of the same instance together while pushing two views of distinct instances apart. As a result, the learned representation is invariant to the transformations in $\mathcal{T}$. The output of $g_\theta$ is L$_2$ normalized to prevent trivial solutions. That is, it is optimal to uniformly disperse representations across the surface of the hypersphere.
A drawback to NCE, and consequently to this class of contrastive objectives, is that the number of negative examples $k$ must be large to faithfully approximate the true partition function.
However, the size of $k$ in Eq.~\ref{eq:contrastive} is limited by compute and memory when optimizing. Thus, recent innovations have focused on tackling this challenge.
% \ndg{can compress the rest of this section (IR, CMC, MoCo) to one paragraph with no eqns? add details in appendix if necessary...}

Instance Discrimination \citep{wu2018unsupervised}, or IR, introduces a memory bank $M$ to cache embeddings of each $x_i \in \mathcal{D}$. Since every epoch we observe each instance once, the memory bank will save the embedding of the view of $x_i$ observed last epoch in its $i$-th entry. Then, the objective is:
\begin{equation}
    \mathcal{L}_{\text{IR}}(x_i; M) = \mathbf{E}_{t \sim p(t)} \mathbf{E}_{j_{1:k} \sim U(1,n)} \left[ \log \frac{e^{g_\theta(t(x_i))^T M[i] / \tau}}{\frac{1}{k+1}\sum_{j\in \{i,j_{1:k}\}} e^{g_\theta(t(x_i))^T M[j] / \tau}} \right]
\label{eq:ir}
\end{equation}
where $M[i]$ represents fetching the $i$-th entry in $M$, and $j_{1:k} \sim U(1,n)$ indicates uniformly sampling $k$ integers from 1 to $n$, or equivalently entries from $M$. Observe that sampling uniformly $k$ times from $M$ is equivalent to $x_{1:k} \sim p(x)$. Representations stored in the memory bank are removed from the automatic differentiation tape, but in return, we can choose a large $k$. Several later approaches \citep{zhuang2019local,tian2019contrastive,iscen2019label,srinivas2020curl} build on the IR framework. In particular, Contrastive Multiview Coding \citep{tian2019contrastive}, or CMC, decomposes an image into luminance and AB-color channels. Then, CMC is the sum of two IR losses where the memory banks for each ``modality'' are swapped, encouraging the representation of the luminance of an image to be ``close'' to the representation of the AB-color of that image, and vice versa.
% :
% \begin{equation}
%     \mathcal{L}_{\text{CMC}}(x^{\text{L}}_i, x^{\text{ab}}_i; M^{\text{L}}, M^{\text{ab}}) = \mathcal{L}_{\text{IR}}(x^{\text{L}}_i; M^{\text{ab}}) + \mathcal{L}_{\text{IR}}(x^{\text{ab}}_i; M^{\text{L}})
% \label{eq:cmc}
% \end{equation}
% As the two modalities capture disjoint types of information, CMC builds strong embeddings that reportedly outperform those of IR on transfer tasks \citep{tian2019contrastive}.

Momentum Contrast \citep{he2019momentum,chen2020improved}, or MoCo, observed that the representations stored in the memory bank grow stale, since possibly thousands of optimization steps pass before updating an entry twice. This is problematic as stale entries could bias gradients. So, MoCo makes two important changes to the IR framework. First, it replaces the memory bank with a first-in first-out (FIFO) queue $Q$ of size $k$. During each minibatch, representations are cached into the queue while the most stale ones are removed. Since elements in a minibatch are chosen i.i.d. from $p(x)$, using the queue as negatives is equivalent to drawing $x_{1:k} \sim p(x)$ i.i.d. Second, MoCo introduces a second (momentum) encoder $g'_{\theta'}$. Now, the primary encoder $g_\theta$ is used to embed one view of instance $x_i$ whereas the momentum encoder is used to embed the other. Again, gradients are not propagated to $g'_{\theta'}$. Instead, its parameters are deterministically set by $\theta' = m \theta' + (1-m)\theta$ where $m$ is a ``momentum'' coefficient. In summary, the MoCo objective is
\begin{equation}
    \mathcal{L}_{\text{MoCo}}(x_i; Q) = \mathbf{E}_{t \sim p(t)} \mathbf{E}_{j \sim U(1,n)} \left[ \log \frac{e^{g_\theta(t(x_i))^T g'_{\theta'}(t'(x_i)) / \tau}}{\frac{1}{k+1}\sum_{j\in \{i,1:k\}} e^{g_\theta(t(x_i))^T Q[j] / \tau}} \right],
\label{eq:moco}
\end{equation}
again equivalent to NCE under a slightly different implementation.
% The third approach, SimCLR \citep{chen2020simple}, removes the need for a bank or queue entirely, cleverly leveraging the other instances in the same minibatch as negative samples. As we are already computing embeddings for all examples in the minibatch, this approach is computationally efficient. Letting $b$ be the minibatch size, consider the objective:
% \begin{equation}
%     \mathcal{L}_{\text{SimCLR}}(x_i) = \mathbf{E}_{t_{1:2b} \sim p(t)}\mathbf{E}_{x_{1:b} \sim p(x)}\left[\frac{e^{g_\theta(t_i(x_j))^T g_\theta(t_{i+b}(x_{i+b}))/\tau}}{\frac{1}{2b-1}\sum_{j\neq i,j=1}^{2b} e^{g_\theta(t_i(x_i))^T g_\theta(t_j(x_j)) / \tau}}\right]
% \label{eq:simclr}
% \end{equation}
% where $x_{j+b}$ is a clone of $x_j$ for $j \in [b]$. In practice, the effectiveness of SimCLR hinges on $b$ being large as this implies more negative examples. While \cite{chen2020simple} found outstanding performance with batch sizes of up to 8192, doing so demands state-of-the-art hardware. An open challenge is to get the benefits of the SimCLR approach with smaller batch sizes.
% \ndg{merge this subsection with one above? and maybe put MI / NCE as first eqn, then specifying form of $f_\theta(u_i,v_i)$ that gives rise to IR etc?}

\section{Conditional Noise Contrastive Estimation}
% talk bout how negative samples are chosen from the marginal distribution and this is primarily done for poole's proof
% talk about how intuitively, we may not want to draw from the marginal and the majority of those examples are far too easy
% we could do regular negative mining tricks but we want it important to preserve the ocnnection of these things to NCE.
In NCE, the negative examples are sampled i.i.d. from the marginal distribution. Indeed, the existing proof that NCE lower bounds mutual information \citep{poole2019variational} assumes this to be true. However, choosing negatives in this manner may not be the best choice for learning a good representation. For instance, prior work in metric learning has shown the effectiveness of hard negative mining in optimizing triplet losses \citep{wu2017sampling,yuan2017hard,schroff2015facenet}. In this work, we similarly wish to exploit choosing negatives \textit{conditional on the current instance} but to do so in a manner that preserves the relationship of contrastive algorithms to mutual information.

We consider the general case of two random variables $U$ and $V$ according to a distribution $p(u,v)$, although the application to the contrastive setting is straightforward.
To start, suppose we define a new distribution $q(v|v^*)$ conditional on a specific realization $v^*$ of $V$. Ideally, we would like for $q(v|v^*)$ to belong to any distribution family but not all choices of $q$ preserve a bound. We provide a counterexample in the Appendix.
This does not, however, imply that we can only sample negatives from the marginal $p(v)$ \citep{poole2019variational,oord2018representation}. One of our theoretical contributions is to formally define a family of conditional distributions $\mathcal{Q}$ such that for any $q \in \mathcal{Q}$, we can draw negative examples from it, instead of $p$, in the NCE estimator while maintaining a lower bound on $\mathcal{I}(U;V)$. We call our new lower bound on mutual information the Conditional Noise Contrastive Estimator, or CNCE.
The next Theorem shows CNCE to bound $\mathcal{I}(U;V)$.
% Although we present this Variational Contrastive Estimator, or VCE, in the general case of two random variables $U$ and $V$, we note that it applies immediately to the contrastive setting where $U$ and $V$ are both views of the same random variable $X$.

% We highlight two special cases of Lemma~\ref{lem:vince} when $T = [\tau, \infty)$ and $T = [\tau, \gamma]$ for some $\gamma > \tau$.
% Next, we define the variational InfoNCE, or VINCE, estimator and  use Lemma~\ref{lem:vince} to show that it lower-bounds
% InfoNCE, and thus mutual information.
% \ndg{and then this is the thm:}
\begin{thm}
Define $U$ and $V_1$ by $p(u,v_1)$ and let $V_1,...,V_k$ be i.i.d. Fix any $f: (U, V_j)\rightarrow \mathbf{R}$ and put $c = \mathbf{E}_{p(v_1)}[e^{f(u, v_1)} ]$. Pick $B\subset \mathbf{R}$ strictly lower-bounded by $c$. Assume $p(S_B)> 0$ for $S_B = \{v | e^{f(u, v)}  \in B\}$. For Borel $A = A_2 \times....\times A_k$, put $q(V_{2:k} \in A) = \prod_{j=2}^k  p(A_j | S_B)$. Let $\mathcal{L}_{\text{CNCE}}(U; V_1) =  \mathbf{E}_{u, v_1 \sim p(u,v_1)}\mathbf{E}_{v_{2:k}\sim q}\left[\log \frac{e^{f(u, v_1)}}{\frac{1}{k}\sum_{j=1}^k e^{f(u, v_j)}}\right]$. Then $\mathcal{L}_{\text{CNCE}} \leq \mathcal{L}_{\text{NCE}} \leq \mathcal{I}(U;V_1)$.
\label{thm:vince}
\end{thm}
\begin{proof}
To show $\mathcal{L}_{\text{CNCE}} \leq \mathcal{L}_{\text{NCE}}$, we show $\mathbf{E}_p[\log(\sum_{j=1}^k e^{f(u,v_j} ) ] < \mathbf{E}_q[\log(\sum_{j=1}^k e^{f(u,v_j} ) ]$. To see this, apply Jensen's to the left-hand side of $\log \mathbf{E}_p[\sum_{j=1}^k e^{f(u,v_j} ]  < \log \sum_{j=1}^k e^{f(u,v_j)}$, which holds if $v_j \in S_B$ for $j \in [k]$, and then take the expectation $\mathbf{E}_q$ of both sides. The last inequality holds by monoticity $\log$, linearity of expectation, and the fact that $\mathbf{E}_p[ e^{f(u,v_j)} ] \leq e^{f(u,v_j)}$.
\end{proof}

% \ndg{do we need this comment here?}For ease, we will relax notation $q(v_{2:k}|v_1) = q(V_{2:k} \in A)$ from here on out.
% \ndg{revise the intuitive summary to make super clear, since many people will only read this part!}
\textbf{Theorem Intuition.} As a high level summary, although using arbitrary negative distributions in NCE does not bound mutual information, we have found a restricted class of distributions $\mathcal{Q}$ that ``subsets the support of the marginal $p(v)$''. In other words, given some fixed $v^*$, we have defined $q(v|v^*) \in \mathcal{Q}$ to constrain the support of $p$ to a set $S_B$ whose members are ``close'' to $v^*$ as measured by the ``similarity function'' $f$. The remaining probability mass assigned by $p$ to elements outside $S_B$ is renormalized to sum to one (i.e. $p(A_j|S_B) = \frac{p(A_j \cap S_B)}{p(S_B)}$) for $q$ to be well-defined. Intuitively, $q$ cannot change $p$ too much: it must redistribute mass proportionally. The primary distinction then, is the smaller support of $q$, which forces the negatives we sample to be harder for $f$ to distinguish from $v^*$. Thm.~\ref{thm:vince} shows that substituting such a distribution into NCE still bounds mutual information.

Interestingly, we also find that CNCE is a looser bound than NCE, which raises the question: \textit{when is a looser bound ever more useful?}
% \begin{figure}[h!]
%   \centering
%   \includegraphics[width=\textwidth]{images/biasvar.pdf}
%   \caption{A lower variance estimator is more stable when using fewer negative samples $k$.}
%   \label{fig:biasvar}
% \end{figure}
% \subsection{Tradeoff between Bias and Variance}3
In reply, we show that while CNCE is a more biased estimator than NCE, in return it has lower variance. This is natural to expect: because $q$ is the result of shifting $p$ around a smaller support, samples from $q$ have less opportunity to deviate, hence lower variance.
% (with respect to the $k$ negative samples).
\begin{thm}
Define $U$ and $V_1$ by $u,v_1 \sim p(u,v_1)$. Fix $q$ as stated in Theorem~\ref{thm:vince}. Define $Z(v_{2:k}) = \log \frac{e^{f(u, v_1)}}{\frac{1}{k}\sum_{j=1}^k e^{f(u, v_j)}}$. By Theorem~\ref{thm:vince}, $\mathbf{E}_{p(v_{2:k})}[Z]$ and $\mathbf{E}_{q(v_{2:k})}[Z]$ are estimators for the mutual information between $U$ and $V_1$. Suppose that $S_B$ is chosen to ensure $\text{Var}_{q(v_{2:k})}[Z ] \leq  \text{Var}_{\tilde{q}(v_{2:k})}[Z ]$, where $\tilde{q}(A) = p(A| S_B^c)$. Then $\textup{Bias}_p(Z) \leq \textup{Bias}_q(Z)$ and $\textup{Var}_p(Z) \geq \textup{Var}_q(Z) $. That is, sampling $v_{2:k}\sim q$ instead of $p$ trades higher bias for lower variance.
% \ndg{can we just say "let $f$ be as in cor.."?}
% \ndg{need to first assert that $Z_{\pi,k}$ is an estimator of something (otherwise bias isn't well defined)..}
% \ndg{what's up with this $k$ goes to infty restriction? that's too vague... if it's not true for all $k$, can we at least say there exists $K$ st for all $k>K$?}
\label{thm:biasvar}
\end{thm}
The proof is in Sec.~\ref{sec:proof:biasvar}.
Given a good similarity function $f$, the elements inside $S_B$ contain values of $v$ truly similar to the fixed point $u$ as measured by $f$. Thus, the elements outside of $S_B$ occupy a larger range of $f$, and thereby are more varied, satisfying the assumption.
Thm.~\ref{thm:biasvar} provides one answer to our question of looseness.
% Applying VCE to contrastive learning, a small edit can be made to Eq.~\ref{eq:contrastive} to sample $x_{2:k} \sim q(x|t(x_1))$. That is, the VCE distribution conditioned on a view of the current instance $x_1$. We explore applying VCE to IR and MoCo in more detail in Sec.~\ref{sec:ring}.
In stochastic optimization, a lower variance objective may lead to better local optima. For representation learning, using CNCE to sample more difficult negatives may (1) encourage the representation to distinguish fine-grained features  useful in transfer tasks, and (2) provide less noisy gradients.
Thm.~\ref{thm:biasvar} also raises a warning: for a  bad similarity function $f$, such as a randomly initialized neural network, we may not get the benefits of lower variance. We will explore the consequences of this for representation learning in the next section.
\begin{figure}[t!]
  \centering
  \begin{subfigure}[b]{0.24\textwidth}
    \centering
    \includegraphics[width=0.5\textwidth]{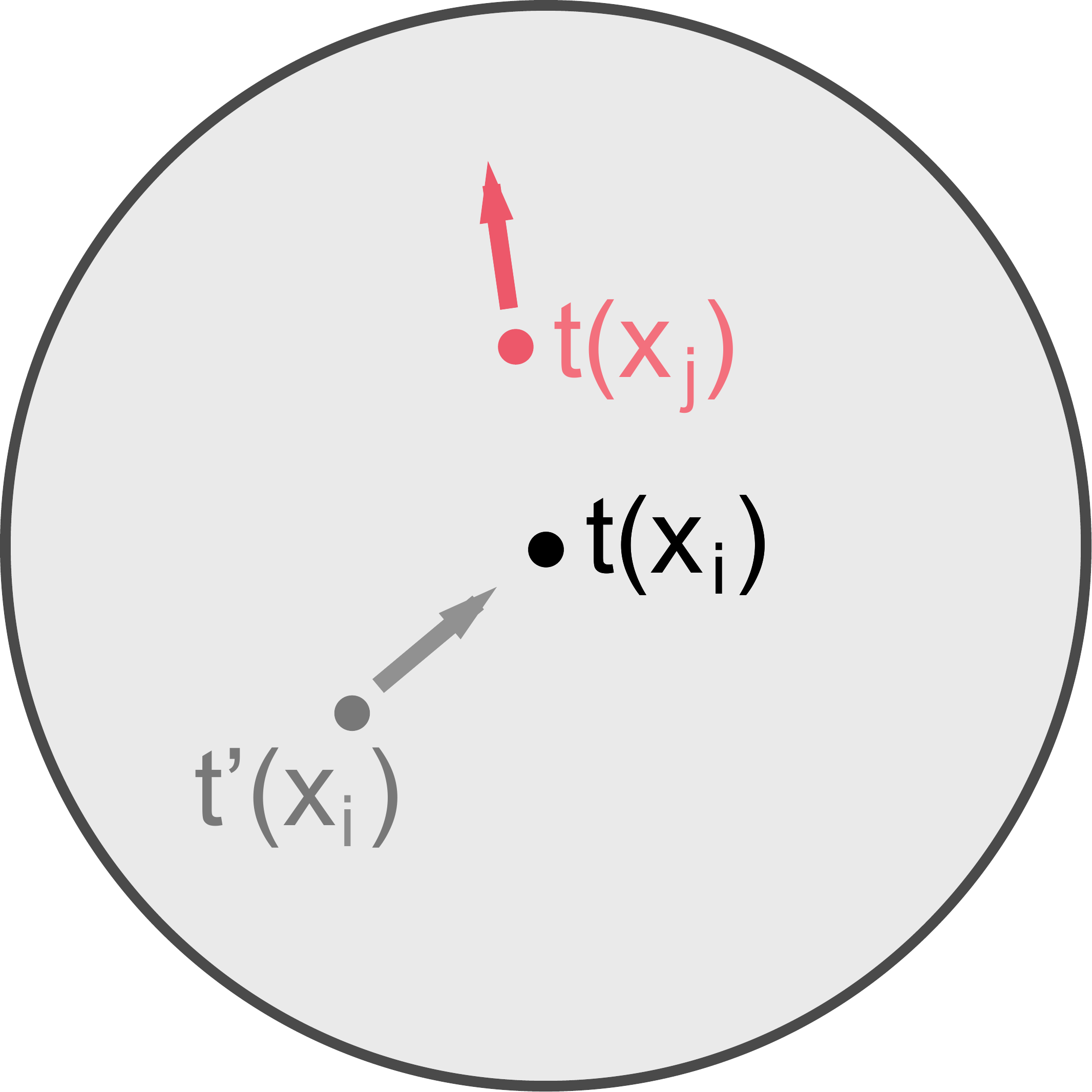}
    \caption{IR, CMC, MoCo}
  \end{subfigure}
  \begin{subfigure}[b]{0.24\textwidth}
    \centering
    \includegraphics[width=0.5\textwidth]{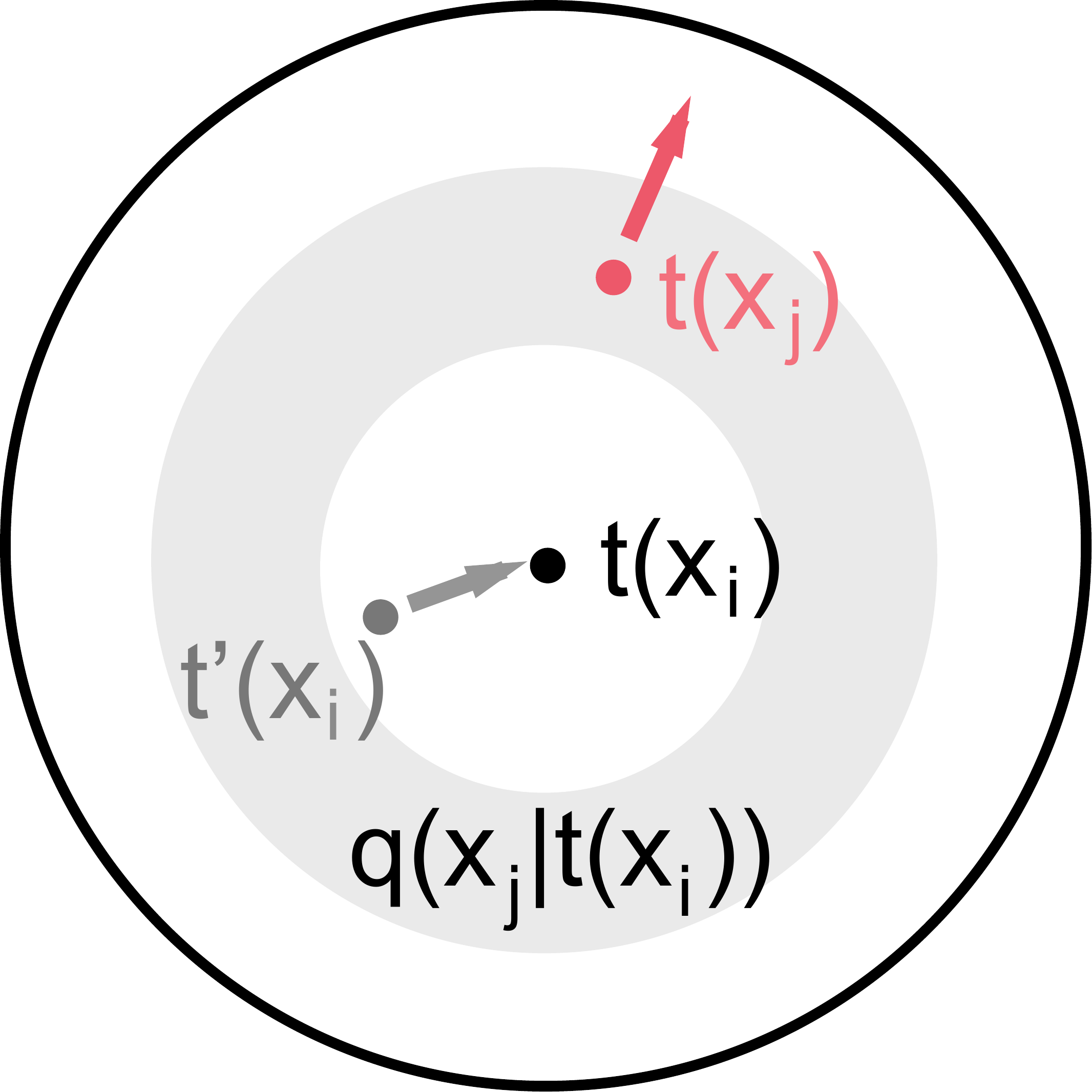}
    \caption{Ring}
  \end{subfigure}
  \begin{subfigure}[b]{0.5\textwidth}
    \centering
    \includegraphics[width=0.8\textwidth]{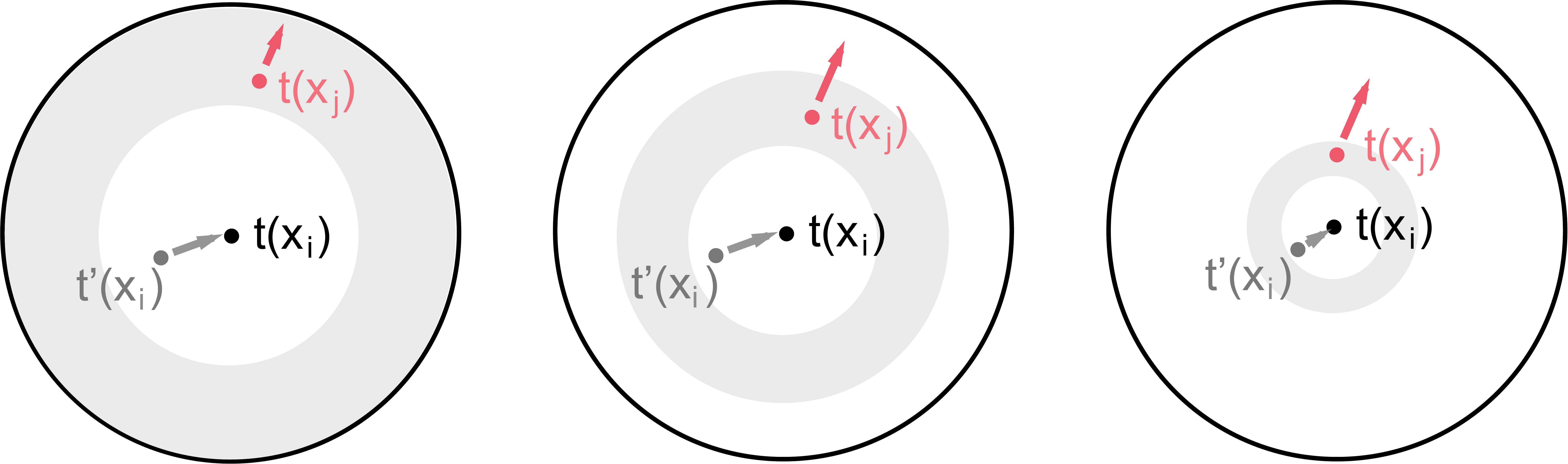}
    \caption{Annealed Ring}
  \end{subfigure}
  \caption{Black: view of instance $x_i$; gray: second view of $x_1$  i.e. the positive example; red: negative samples; gray area: negative distribution $q(x|t(x_i))$. In subfigure (c), the negative samples are annealed to be closer to $t(x_i)$ through training. In other words, the support of $q$ shrinks.}
  \label{fig:models}
\end{figure}
\section{Ring Discrimination}
\label{sec:ring}
We have shown that the CNCE objective provides a lower variance bound on the mutual information between two random variables. Now, we wish to apply CNCE to contrastive learning where the two random variables are derived from two views of a complex random variable $X$.
To do so, we must specify a concrete proposal for the support set $S_B$.
\begin{lrbox}{\codebox}
\begin{lstlisting}[
  language=Python,
  basicstyle=\tiny,
  commentstyle=\color{mygreen},
]
# g_q, g_k: encoder networks
# m: momentum; t: temperature
# omega_u: ring upper threshold
# omega_\ell: ring lower threshold
tx1=aug(x)  # random augmentation
tx2=aug(x)
emb1=norm(g_q(tx1))
emb2=norm(g_k(tx2)).detach()
dps=sum(tx1*tx2)/t  # dot product
# sort from farthest to closest neg
all_dps=sort(emb1@queue.T/t)
# find indices of thresholds
ix_\ell=omega_\ell*len(queue)
ix_u=omega_u*len(queue)
ring_dps=all_dps[:,ix_\ell:ix_u]
# nonparametric softmax
loss=-dps+logsumexp(ring_dps)
loss.backward()
step(g_q.params)
# moco updates
g_k.params = m*g_k.params+\
             (1-m)*g_q.params
enqueue(queue,emb2); dequeue(queue)
# threshold updates
anneal(omega_\ell); anneal(omega_u)
\end{lstlisting}
\end{lrbox}
\begin{wrapfigure}{L}{0.34\textwidth}
  \begin{minipage}{0.34\textwidth}
    \begin{algorithm}[H]
      \usebox{\codebox}
      \caption{MoCoRing}
      \label{fig:algos}
    \end{algorithm}
  \end{minipage}
\end{wrapfigure}
Suppose we take the $i$-th example $x_i \in \mathcal{D}$, and choose a percentile $w_\ell \in [0,100]$. Given the dataset $\mathcal{D}$, we consider each $x$ as a negative example if and only if the normalized distance, $g_\theta(t(x_i))^Tg_\theta(t'(x))$, is above the $w_\ell$-th percentile of all $x \in \mathcal{D}$ for fixed transforms $t,t' \sim p(t)$. That is, we construct $q(x|t(x_i))$ such that we remove examples from the dataset whose representation dot producted with the representation of $t(x_i)$ is ``too low''. (Note that $w_\ell = 0$ recovers Eq.~\ref{eq:contrastive}.) Under this formulation, $w_\ell$ uniquely specifies a set of distances $B$ (recall Thm.~\ref{thm:vince}) no lower than a threshold. For a smaller enough choice of $w_\ell$, this threshold will be greater than expected distance with respect to $p$. In effect, the pre-image set $S_B$ contains all $x \in \mathcal{D}$ whose distance to $t(x_i)$ in representation space is above the $w_\ell$-th percentile.
% We rarely choose $\omega$ explicitly as it is much easier to choose a value for $w_\ell$.
% \ndg{what we actually need is the set $B$ not the bound $\omega$, right?}

However, picking the closest examples to $t(x_i)$ as its negative examples may be inappropriate, as these examples might be better suited as positive views rather than negatives \citep{zhuang2019local,xie2020delving}. As an extreme case, if the same image is included in the dataset twice, we would not like to select it as a negative example for itself.  Furthermore, choosing  negatives ``too close'' to the current instance may result in representations that pick up on fine-grain details only, ignoring larger semantic concepts. For instance, we may find representations that can distinguish two cats based on fur but are unable to classify animals from cars. This suggests removing a set from $S_B$ of instances we consider ``too close'' to the current example. In practice, this translates to picking two percentiles $w_\ell < w_u \in [0,100]$. Now, we consider each example $x$ as a negative example for $x_i$ if and only if $g_\theta(t(x_i))^Tg_\theta(t'(x))$ is within the $w_\ell$-th to $w_u$-th percentiles of all $x \in \mathcal{D}$.
We are free to define the support set $S_B$ in this manner as Thm.~\ref{thm:vince} does not require $S_B$ to contain \textit{all} elements with high similarity to $t(x_i)$.
Intuitively, we construct a conditional distribution for negative examples that are (1) not too easy since their representations are fairly similar to that of $t(x_i)$ and (2) not too hard since we remove the ``closest'' instances to $x_i$ from $S_B$. We call this algorithm \textit{Ring Discrimination}, or Ring, inspired by the shape of negative set (see Fig.~\ref{fig:models}).
% \ndg{need to set up the idea of excluding too similar negatives a bit more..}

Ring can be easily added to popular contrastive algorithms. For IR and CMC, this amounts to simply sampling entries in the memory bank that fall within the $w_\ell$-th to $w_u$-th percentile of  all distances to the current instance view (in representation space). Similarly, for MoCo, we sample from a subset of the queue (chosen to be in the $w_\ell$-th to $w_u$-th percentile), preserving the FIFO ordering. In our experiments, we refer to these as IRing, CMCRing, MoCoRing, respectively. Alg.~\ref{fig:algos} shows PyTorch-like pseudocode for MoCoRing. One of the strengths of this approach is the simplicity: the algorithm requires only a few lines of code on top of existing implementations.

\textbf{Annealing Policy.}
Naively using Ring can collapse to a poor representation, as hinted by Thm.~\ref{thm:biasvar}. Early in training, when the representations are still disorganized, choosing negatives that are close in representation may detrimentally exclude those examples that are ``actually'' close. This could lock in poor local minima. % as the chosen negatives are too difficult for the poor representation.
%Intuitively, we should not expect hard negative mining to \textit{always} result in stronger representations. For example, early in training, when the learned embeddings are poor, we expect using difficult negative samples to slow learning as $g_\theta$ is not competent enough to discriminate between more similar instances. Doing so could result in finding poor local minima.
%However, later in training, when the embedding has already converged to a good mapping of instances to the hypersphere, more difficult negative samples could give rise to better performing embeddings as it forces the model to focus on finer-grain details.
%To capture this intuition,
To avoid this possibility we propose to use Ring with an annealing policy that reduces the size of $S_B$ throughout training. To do this, early in training we choose $w_\ell$ to be small. Over many epochs, we slowly anneal $w_\ell$ to approach $w_u$ thereby selecting more difficult negatives. We explored several annealing policies and found a linear schedule to be well-performing and simple (see Appendix). In our experiments, we found annealing thresholds to be crucial: being too aggressive with negatives early in training resulted in convergence to poor optima.
% We refer to Fig.~\ref{fig:models} for a visualization.

\section{Experiments}
% CIFAR10, CIFAR100, ImageNet
We explore our method applied to IR, CMC, and MoCo in four commonly used visual datasets. As in prior work \citep{wu2018unsupervised,zhuang2019local,he2019momentum,misra2020self,henaff2019data,kolesnikov2019revisiting,donahue2019large,bachman2019learning,tian2019contrastive,chen2020simple}, we evaluate each method by linear classification on frozen embeddings. That is, we optimize a contrastive objective on a pretraining dataset to learn a representation; then, using a transfer dataset, we fit logistic regression on representations only. A better representation would contain more ``object-centric'' information, thereby achieving a higher classification score.

\textbf{Training Details.} We resize input images to be 256 by 256 pixels, and normalize them using dataset mean and standard deviation. The temperature $\tau$ is set to 0.07. We use
a composition of a 224 by 224-pixel random crop, random color jittering, random horizontal flip, and random grayscale conversion as our augmentation family $\mathcal{T}$.
We use a ResNet-18 encoder with a output dimension of 128. For CMC, we use two ResNet-18 encoders, doubling the number of parameters. For linear classification, we treat the pre-pool output (size $512\times 7 \times 7$) after the last convolutional layer as the input to the logistic regression. Note that this setup is equivalent to using a linear projection head \citep{chen2020simple,chen2020improved}.
In pretraining, we use SGD with learning rate 0.03, momentum 0.9 and weight decay 1e-4 for 300 epochs and batch size 256 (128 for CMC). We drop the learning rate twice by a factor of 10 on epochs 200 and 250. In transfer, we use SGD with learning rate 0.01, momentum 0.9, and no weight decay for 100 epochs without dropping learning rate. Future work can explore orthogonal factors such as choice of architecture or pretext task.
% \ndg{some details can move to appendix if we need space?}

\begin{table}[h!]
\tiny
\centering
\begin{subtable}[h]{0.24\textwidth}
    \centering
    \begin{tabular}{ll}
    \toprule
    Model & Transfer Acc. \\
    \midrule
    IR & 81.2 \\
    IRing & 83.9 {\color{Green} (+2.7)} \\
    CMC$^*$ & 85.6 \\
    CMCRing$^*$ & \textbf{87.6} {\color{Green} (+2.0)} \\
    MoCo & 83.1 \\
    MoCoRing & 86.1 {\color{Green} (+3.0)}\\
    LA & 83.9 \\
    % SimCLR & 88.1 \\
    \bottomrule
    \end{tabular}
    \caption{CIFAR10}
\end{subtable}
\begin{subtable}[h]{0.24\textwidth}
    \centering
    \begin{tabular}{ll}
    \toprule
    Model & Transfer Acc. \\
    \midrule
    IR & 60.4 \\
    IRing & \textbf{62.3} {\color{Green} (+1.9)} \\
    CMC$^*$ & 56.0 \\
    CMCRing$^*$ & 56.0 {\color{Gray} (+0.0)}\\
    MoCo & 59.1 \\
    MoCoRing & 61.5 {\color{Green} (+2.4)}\\
    LA & 61.4 \\
    % SimCLR & \\
    \bottomrule
    \end{tabular}
    \caption{CIFAR100}
\end{subtable}
\begin{subtable}[h]{0.24\textwidth}
    \centering
    \begin{tabular}{ll}
    \toprule
    Model & Transfer Acc. \\
    \midrule
    IR & 61.4 \\
    IRing & 64.3 {\color{Green} (+2.9)} \\
    CMC$^*$ & 63.8 \\
    CMCRing$^*$ & \textbf{66.4} {\color{Green} (+2.6)} \\
    MoCo & 63.8 \\
    MoCoRing & 65.2 {\color{Green} (+1.4)} \\
    LA & 63.0 \\
    % SimCLR & \\
    \bottomrule
    \end{tabular}
    \caption{STL10}
\end{subtable}
\begin{subtable}[h]{0.24\textwidth}
    \centering
    \begin{tabular}{ll}
    \toprule
    Model & Transfer Acc. \\
    \midrule
    IR & 43.2 \\
    IRing & 48.4 {\color{Green} (+5.2)} \\
    CMC$^*$ & 48.2 \\
    CMCRing$^*$ & \textbf{50.4} {\color{Green} (+2.2)} \\
    MoCo & 52.8 \\
    MoCoRing & 54.6 {\color{Green} (+1.8)} \\
    LA & 48.0 \\
    % SimCLR & \\
    \bottomrule
    \end{tabular}
    \caption{ImageNet}
\end{subtable}
\caption{Comparison of contrastive algorithms on three  image domains. Superscript ($^*$) indicates models that use twice as many parameters as others e.g. CMC has ``L'' and ``ab'' encoders.}
\label{table:results1}
\end{table}

The results for CIFAR10, CIFAR100, STL10, and ImageNet are in Table~\ref{table:results1}. Overall, IR, CMC, and MoCo all benefit from using more difficult negatives as shown by 2-5\% absolute points of improvement across the four datasets. While we find different contrastive objectives to perform best in each dataset, the improvements from Ring are consistent: the Ring variant outperforms the base for every model and every dataset. We also include as a baseline Local Aggregation, or LA \citep{zhuang2019local}, a popular contrastive algorithm (see Sec.~\ref{sec:la}) that implicitly uses hard negatives without annealing. We find our methods to outperform LA by up to 4\% absolute.

\begin{wraptable}{l}{4.5cm}
\tiny
\centering
\begin{subtable}[h]{0.24\textwidth}
    \centering
    \begin{tabular}{ll}
    \toprule
    Model & Acc. \\
    \midrule
    IR & 81.2 \\
    IRing & 83.9 \\
    IRing (No Anneal) & 81.4 \\
    IRing ($w_u = 100$) & 82.1 \\
    % SimCLR & 88.1 \\
    \bottomrule
    \end{tabular}
    \caption{CIFAR10}
\end{subtable}
\begin{subtable}[h]{0.24\textwidth}
    \centering
    \begin{tabular}{ll}
    \toprule
    Model & Acc. \\
    \midrule
    IR & 43.2 \\
    IRing & 48.4 \\
    IRing (No Anneal) & 41.3 \\
    IRing ($w_u = 100$) & 47.3 \\
    % SimCLR & 88.1 \\
    \bottomrule
    \end{tabular}
    \caption{ImageNet}
\end{subtable}
\caption{Lesioning the effects of annealing and choice of $w_u$.}
\label{table:lesion}
\end{wraptable}
\textbf{Lesions: Annealing and Upper Boundary.}
Having found good performance with Ring Discrimination, we want to assess the importance of the individual components that comprise Ring. We focus on the annealing policy and the exclusion of very close negatives from $S_B$.
Concretely, we measure the transfer accuracy of (1) IRing without annealing and (2) IRing with an upper percentile $w_u$ set to 100, thereby excluding no close negatives. That is, $S_B$ contains \textit{all} examples in the dataset with representation similarity greater than the $w_\ell$-th percentile (a ``ball'' instead of a ``ring''). Table~\ref{table:lesion} compares these lesions to IR and full IRing on CIFAR10 and ImageNet classification transfer. We observe that both lesions result in worse transfer accuracy, with proper annealing being especially important, confirming the suspicions raised by Thm.~\ref{thm:biasvar}.
%
% with around a 1-2 point difference when using $w_u = 100$ and a larger 3-7 point different without annealing. In ImageNet, IRing without annealing under-performs IR, showing evidence that overly-difficult negatives face challenges with local minima.
% In practice, we recommend annealing and setting $w_u < 100$.

\textbf{Transferring Features.}
Thus far we have only evaluated the learned representations on unseen examples from the training distribution. As the goal of unsupervised learning is to capture \textit{general} representations, we are also interested in their performance on new, unseen distributions. To gauge this, we use the same linear classification paradigm on a suite of image datasets from the ``Meta Dataset'' collection \citep{triantafillou2019meta} that have been used before in contrastive literature \citep{chen2020simple}. All representations were trained on CIFAR10. For each transfer dataset, we compute mean and variance from a training split to normalize input images, which we found important for generalization to new visual domains.
% All other hyperparameters are as described above.
\begin{table}[h!]
    \tiny
    \centering
    \begin{tabular}{llllllllll}
    \toprule
    Model & Aircraft & CUBirds & DTD & Fungi & MNIST & FashionMNIST & TrafficSign & VGGFlower & MSCOCO \\
    \midrule
    IR & 40.9 & 17.9 & 39.2 & 2.7 & 96.9 & 91.7 & 97.1 & 68.1 & 52.4 \\
    IRing & 40.6 {\color{Red} (-0.3)} & 17.9 {\color{Gray} (+0.0)} & 39.5 {\color{Gray} (+0.3)} & 3.4 {\color{Green} (+0.7)} & 97.8 {\color{Green} (+0.9)} & 91.6 {\color{Gray} (+0.1)} & 98.8 {\color{Green} (+1.7)} & 68.5 {\color{Gray} (+0.4)} & 52.5 {\color{Gray} (+0.1)} \\
    MoCo & 41.5 & 18.0 & 39.7 & 3.1 & 96.9 & 90.9 & 97.3 & 64.5 & 52.0 \\
    MoCoRing & 41.6{\color{Gray} (+0.1)} & 18.6 {\color{Green} (+0.6)} & 39.5 {\color{Red} (-0.2)} & 3.6 {\color{Green} (+0.5)} & 97.9 {\color{Green} (+1.0)} & 91.3 {\color{Gray} (+0.4)} & 99.3 {\color{Green} (+2.0)} & 69.1 {\color{Green} (+4.6)} & 52.6 {\color{Green} (+0.6)}\\
    CMC & 40.1 & 15.8 & 38.3 & 4.3 & 97.5 & 91.5 & 94.6 & 67.1 &  51.4 \\
    CMCRing & 40.8 {\color{Green} (+0.7)} & 16.8 {\color{Green} (+1.0)} & 40.6 {\color{Green} (+2.3)} & 4.2 {\color{Red} (-0.1)} & 97.9 {\color{Gray} (+0.4)} & 92.1 {\color{Green} (+0.6)} & 97.1 {\color{Green} (+2.5)} & 69.1 {\color{Green} (+2.0)} & 52.1 {\color{Green} (+0.7)} \\
    LA & 41.3 & 17.8 & 39.0 & 2.3 & 97.2 & 92.3 & 98.2 & 66.9 & 52.3 \\
    \bottomrule
    \end{tabular}
    \caption{Transferring CIFAR10 embeddings to various image distributions.}
    \label{table:meta}
\end{table}

We find in Table~\ref{table:meta} that the Ring models are competitive with the non-Ring analogues, with increases in transfer accuracies of 0.5 to 2\% absolute. Most notable are the TrafficSign and VGGFlower datasets in which Ring models surpass others by a larger margin. We also observe that IRing largely outperforms LA. This suggests the features learned with more difficult negatives are not only useful for the training distribution but may also be transferrable to many visual datasets.

\textbf{More Downstream Tasks.}
Object classification is a popular transfer task, but we want our learned representations to capture holistic knowledge about the contents of an image. We must thus evaluate performance on transfer tasks such as detection and segmentation that require different kinds of visual information. We study four additional downstream tasks: object detection on COCO \citep{lin2014microsoft} and Pascal VOC'07 \citep{everingham2010pascal}, instance segmentation on COCO, and keypoint detection on COCO. In all cases, we employ embeddings trained on ImageNet with a ResNet-18 encoder. We base these experiments after those found in \cite{he2019momentum} with the same hyperparameters. However, we use a smaller backbone (ResNet-18 versus ResNet-50) and we freeze its parameters instead of finetuning them. We adapt code from Detectron2 \citep{wu2019detectron2}.
% See appendix for specific details.

\begin{table}[h!]
\tiny
\centering
\begin{tabular}{l|lll|lll|lll|lll}
\toprule
 & \multicolumn{3}{c}{COCO: Object Detection} & \multicolumn{3}{c}{COCO: Inst. Segmentation} & \multicolumn{3}{c}{COCO: Keypoint Detection} & \multicolumn{3}{c}{VOC: Object Detection} \\
\midrule
Arch. & \multicolumn{6}{c}{Mask R-CNN, R$_{18}$-FPN, 1x schedule} & \multicolumn{3}{c}{R-CNN, R$_{18}$-FPN} & \multicolumn{3}{c}{Faster R-CNN, R$_{18}$-C4} \\
 \midrule
 Model & AP$^{\text{bb}}$ & AP$_{50}^{\text{bb}}$ & AP$_{75}^{\text{bb}}$ & AP$^{\text{mk}}$ & AP$_{50}^{\text{mk}}$ & AP$_{75}^{\text{mk}}$ & AP$^{\text{kp}}$ & AP$_{50}^{\text{kp}}$ & AP$_{75}^{\text{kp}}$ & AP$^{\text{bb}}$ & AP$_{50}^{\text{bb}}$ & AP$_{75}^{\text{bb}}$ \\
 \midrule
 IR & 8.6 & 19.0 & 6.6 & 8.5 & 17.4 & 7.4 & 34.6 & 63.0 & 32.9 & 5.5 & 14.5 & 3.3 \\
 IRing & \textbf{10.9} & \textbf{22.9} & \textbf{8.7} & 11.0 & 20.9 & 9.6 & 37.2 & 66.1 & 35.7 & 7.6 & 20.3 & 4.4 \\
 MoCo & 6.0 & 14.3 & 4.0 & 10.8 & 21.4 & 9.7 & 37.6 & 66.5 & 36.9 & 7.3 & 17.9 & 4.1 \\
 MoCoRing & 9.4 & 20.3 & 7.6 & \textbf{12.0} & \textbf{22.9} & \textbf{10.8} & \textbf{38.7} & \textbf{67.7} & \textbf{37.9} & \textbf{8.0} & \textbf{22.1} & \textbf{4.8} \\
 LA & 10.2 & 22.0 & 8.1 & 10.0 & 20.3 & 9.0 & 36.3 & 65.3 & 35.1 & 7.6 & 20.0 & 4.3 \\
\bottomrule
\end{tabular}
\caption{Evaluation of ImageNet representations using four visual transfer tasks.}
\label{tab:imagenet}
\end{table}
We find IRing outperforms IR by around 2.3 points in COCO object detection, 2.5 points in COCO Instance Segmentation, 2.6 points in COCO keypoint detection, and 2.1 points in VOC object detection. Similarly, MoCoRing finds consistent improvements of 1-3 points over MoCo on the four tasks. Future work can investigate orthogonal directions of using larger encoders (e.g. ResNet-50) and finetuning ResNet parameters for these individual tasks.

\section{Related Work}
\label{sec:related}
% 1. talk about negative mining in other contrastive frameworks and how they do not connect to mutual info
% 2. talk about LA and how it already employs HNM
% 3. mention MI, poole, paper for MI. mention how our estimator is different than the maxview paper.
% 4. mention byol and how future work can think about that carefully
Several of the ideas in Ring Discrimination relate to existing work. Below, we explore these connections, and at the same time, place our work in a fast-paced and growing field.

\textbf{Hard negative mining.}
While it has not been deeply explored in modern contrastive learning, negative mining has a rich line of research in the metric learning community. Deep metric learning utilizes triplet objectives of the form $\mathcal{L}_{\text{triplet}} = d(g_\theta(x_i), g_\theta(x_+)) - d(g_\theta(x_i), g_\theta(x_-) + \alpha)$ where $d$ is a distance function (e.g. L$_2$ distance), $x_+$ and $x_-$ are a positive and negative example, respectively, relative to $x_i$, the current instance, and $\alpha \in \mathbf{R}^+$ is a margin.
In this context, several approaches pick semi-hard negatives: \cite{schroff2015facenet} treats the furthest (in L$_2$ distance) example in the same minibatch as $x_i$ as its negative, whereas \cite{oh2016deep} weight each example in the minibatch by its distance to $g_\theta(x_i)$, thereby being a continuous version of \cite{schroff2015facenet}.
More sophisticated negative sampling strategies developed over time.
In \cite{wu2017sampling}, the authors pick negatives from a fixed normal distribution that is shown to  approximate L$_2$ normalized embeddings in high dimensions. The authors show that weighting by this distribution samples more diverse negatives.
Similarly, HDC \citep{yuan2017hard} simulataneously optimizes a triplet loss using many levels of ``hardness'' in negatives, again improving the diversity. Although triplet objectives paved the way for modern NCE-based objectives, the focus on negative mining has largely been overlooked. Ring Discrimination, being inspired by the deep metric learning literature, reminds that negative sampling is still an effective way of learning stronger representations in the new NCE framework. As such, an important contribution was to do so while retaining the theoretical properties of NCE, namely in relation to mutual information. This, to the best of our knowledge, is novel as negative mining in metric learning literature was not characterized in terms of information theory.

That being said, there are some cases of negative mining in contrastive literature. In CPC \citep{oord2018representation}, the authors explore using negatives from the same speaker versus from mixed speakers in audio applications, the former of which can be interpreted as being more difficult. A recent paper, InterCLR \citep{xie2020delving}, also finds that using ``semi-hard negatives'' is beneficial to contrastive learning whereas negatives that are too difficult or too easy produce worse representations. Where InterCLR uses a margin-based approach to sample negatives, we explore a wider family of negative distributions and show analysis that annealing offers a simple and easy solution to choosing between easy and hard negatives. Further, as InterCLR's negative sampling procedure is a special case of CNCE, we provide theory grounding these approaches in information theory. Finally, a separate line of work in contrastive learning explores using neighboring examples (in embedding space) as ``positive'' views of the instance \citep{zhuang2019local,xie2020delving,asano2019self,caron2020unsupervised,li2020prototypical}. That is, finding a set $\{x_j\}$ such that we consider $x_j = t(x_i)$ for the current instance $x_i$.
While this does not deal with negatives explicitly, it shares similarities to our approach by employing   other examples in the contrastive objective to learn better representations. In the Appendix, we discuss how one of these algorithms, LA \citep{zhuang2019local}, implicitly uses hard negatives and expand the Ring family with ideas inspired by it.

\textbf{Contrastive learning.} We focused primarily on comparing Ring Discrimination to three recent and highly performing contrastive algorithms, but the field contains much more. The basic idea of learning representations to be invariant under a family of transformations is an old one, having been explored with self-organizing maps \citep{becker1992self} and dimensionality reduction \citep{hadsell2006dimensionality}. Before IR, the idea of instance discrimination was studied \citep{dosovitskiy2014discriminative,wang2015unsupervised} among many pretext objectives such as position prediction \citep{doersch2015unsupervised}, color prediction \citep{zhang2016colorful}, multi-task objectives \citep{doersch2017multi}, rotation prediction \citep{gidaris2018unsupervised,chen2019self}, and many other ``pretext'' objectives \citep{pathak2017learning}.
As we have mentioned, one of the primary challenges to instance discrimination is making such a large softmax objective tractable. Moving from a parametric \citep{dosovitskiy2014discriminative} to a nonparametric softmax reduced issues with vanishing gradients, shifting the challenge to efficient negative sampling. The memory bank approach \citep{wu2018unsupervised} is a simple and memory-efficient solution, quickly being adopted by the research community \citep{zhuang2019local,tian2019contrastive,he2019momentum,chen2020improved,misra2020self}. With enough computational resources, it is now also possible to reuse examples in a large minibatch and negatives of one another \citep{ye2019unsupervised,ji2019invariant,chen2020simple}.
In our work, we focus on hard negative mining in the context of a memory bank or queue due to its computational efficiency. However, the same principles should be applicable to batch-based methods (e.g. SimCLR): assuming a large enough batch size, for each example, we only use a subset of the minibatch as negatives as in Ring.
Finally, more recent work \citep{grill2020bootstrap}  removes negatives altogether, which is speculated to implicitly use negative samples via batch normalization \citep{ioffe2015batch}. We leave a more thorough understanding of negatives in BYOL to future work.

\section{Conclusion}
In this work, we presented a family of mutual information estimators that approximate the partition function using samples from a class of conditional distributions. We proved several theoretical statements about this family, showing a bound on mutual information and a tradeoff between bias and variance. Then, we applied these estimators as objectives in contrastive representation learning.
In doing so, we found that our representations outperform existing approaches consistently across a spectrum of contrastive objectives, data distributions, and transfer tasks.
% The ideas from Ring Discrimination can be applied easily and broadly to many popular contrastive frameworks.
Overall, we hope our work to encourage more exploration of negative sampling in the recent growth of research in contrastive learning. Future work can investigate better annealing protocols to ensure diversity.

\bibliography{iclr2020_conference}
\bibliographystyle{iclr2020_conference}

\newpage
\appendix

\section{Proofs}
\label{sec:proofs}

\subsection{Counterexample Against Unrestricted $q$.}
Pick some $x^* \sim p(x)$ and take $f$ to be a continuous function whose range spans $[0,1]$. For any $\epsilon > 0$, pick $q$ to be a distribution such that for every $x \sim q$ with non-zero probability, we have $f(x, x^*) < \epsilon$. Then, by varying $\epsilon$ closer to 0, we can bring our bound on mutual information to infinity, regardless of the true value, thus ceasing to be a bound. As such, we cannot use an unrestricted family of conditional distributions and preserve a bound.

\subsection{Proof of Theorem~\ref{thm:biasvar}}
\label{sec:proof:biasvar}
\begin{proof}
We separately show the statements regarding bias and variance.

First, as $Z(v_{2:k})$ with $v_{2:k}\sim p(v_{2:k})$ (or NCE) is unbiased, and $\mathbf{E}_{q(v_{2:k})}[Z]$ (or CNCE) lower bounds $\mathbf{E}_{p(v_{2:k})}[Z]$ (Thm.~\ref{thm:vince}), the first statement follows immediately for any choice of $k$.

Second, by the law of total variance,
\[
 \mathbf{E}_p[\text{Var}_p[Z |\mathbf{1}_{S_B } ] ]+  \text{Var}_p(\mathbf{E}_p[Z |\mathbf{1}_{S_B }]) = \text{Var}_p[Z ]
\]
Since both summands are non-negative and the variance on the right is the desired upper-bound, it suffices to show that
\[
p(S_B) \cdot \text{Var}_{q(v_{2:k})}[Z  ] \leq  \mathbf{E}_{p} [\text{Var}_p[Z |\mathbf{1}_{S_B }] ].
\]
This follows immediately from the observation that by definition of $q(\cdot)$ as the conditional distribution $p(\cdot | S_B)$, the expectation on the right is precisely
\[
p(S_B) \cdot \text{Var}_{q(v_{2:k})}[Z ] + (1 - p(S_B)) \cdot  \text{Var}_{\tilde{q}(v_{2:k})}[Z ],
\]
where $\tilde{q}$ is the conditional distribution $p(\cdot | \neg S_B)$.
\end{proof}

% We conclude with a Lemma used in the proof above.
% \begin{lem}
% Let $X_{p,k}$ be a random variable representing a mean computed from $k$ samples drawn from a distribution $p$, and let $\mu$ represent the true mean with respect to $p$. Picking $h = \max(\log, 1)$, then  $\mathbf{E}_p[h(X_{p,k})]$ approaches $h(\mu)$ as the number of samples $k$ increases.
% \label{lem:milan}
% \end{lem}
% \begin{proof}
% By the Strong Law, $X_{p,k}$ converges almost surely to $\mu$. Thus, for $h$ continuous and bounded, $h(X_{p,k})$ converges almost surely to $h(\mu)$ and $\lim_{k \rightarrow \infty} \mathbf{E}_p \left[h(X_{p,k})\right] = \mathbf{E}_p\left[h(\mu)\right]$.\newline

% \end{proof}

\section{A Toy Example}
\label{sup:toy}
Interestingly, Thm.~\ref{thm:vince} shows CNCE to lower bound NCE. To confirm this experimentally, we re-purpose the toy setting from \cite{tschannen2019mutual}.
Pick two random variables $Z$ and $\epsilon$ distributed such that $z_i \sim \mathcal{N}(0, \Sigma_Z)$ and $\epsilon_i \sim \mathcal{N}(0, \Sigma_\epsilon)$ where $\Sigma_Z = \begin{pmatrix}
1 & -0.5 \\
-0.5  & 1
\end{pmatrix}$ and $\Sigma_\epsilon = \begin{pmatrix}
1 & 0.9 \\
0.9 & 1
\end{pmatrix}$
Then, let $(X, Y) = Z + \epsilon$. That is, let $X$ be the first dimension of the sum and $Y$ the second. The mutual information between $X$ and $Y$ can be analytically computed as
$-\frac{1}{2}\log (1- \frac{\Sigma[1,2]\Sigma[2,1]}{\Sigma[1,1]\Sigma[2,2]})$
 since $(X, Y)$ is jointly Gaussian with covariance $\Sigma = \Sigma_Z + \Sigma_\epsilon$.
 \begin{table}[h!]
  \tiny
  \centering
    \begin{tabular}{lcccccccc}
        \toprule
        & True & NCE & \multicolumn{6}{c}{CNCE} \\
        \midrule
        $\omega$ & & & 10 & 25 & 50 & 75 & 90 & 95 \\
        \midrule
        Mean & 0.02041 & 0.01345 & 0.01241   & 0.00220  & 7.29e-5   & 1.67e-5 & 5.87e-6  & 1.97e-6 \\
        Stdev & -- & 0.001 & 3e-4 & 1e-4 & 9e-6 & 2e-6 & 1e-6 & 4e-6 \\
        \bottomrule
    \end{tabular}
    \caption{Looseness of CNCE as $w_\ell$ increases.}
 \label{table:toy1}
\end{table}
For this toy experiment, let $w_\ell$ be a percentage from 0 to 100. Now, we define $S_B$ as all examples whose dot product with the embedding of the current transformed instance is in the top $w_\ell$ percentage of all examples in the dataset. We can tractably compute this using a memory bank.
% \ndg{how is $S_B$ defined here? presumably a ball?}
As $w_\ell$ increases from 10 to 95,
% \ndg{indicate units of $\omega$.. percent of data, but i don't think that's been said..},
$q(x|t(x_i))$ has smaller support meaning that negative samples are more difficult to separate from the current instance $x_i$.
% For simplicitly, we let $w_u = 100$ and fit VCE with different values $w_\ell$ ranging from $10$ to $95$ (higher values for $w_\ell$ indicate harder negative samples).
Table~\ref{table:toy1} compares the estimated mutual information between $X$ and $Y$ from each estimator to the ground truth over 5 runs.
The encoders are 5-layer MLPs with 10 hidden dimensions and ReLU nonlinearities. To build the dataset, we sample 2000 points and optimize the NCE objective with Adam with a learning rate of 0.03, batch size 128, and no weight decay for 100 epochs. Given a percentage for CNCE, we compute distances between all elements in the memory bank and the representation the current image --- we only sample 100 negatives from the top $p$ percent. We conduct the experiment with 5 different random seeds.

\section{Bias and Variance Experiment Details}
For IR, we explore $k = 16,32,64,128,256,512,1024,4096$. For MoCo, we only evaluate $k =256,512,1024,4096$ as the queue cannot be smaller than the batch size. All hyperparameter choices are as detailed in the main experiments. To find the nearest neighbor of the training example, we store all embeddings in a memory bank (separate from the one possibly used in training).

\section{Detectron2 Experiments}
\label{sec:detectron}
We make heavy usage of the Detectron2 code found at \url{https://github.com/facebookresearch/detectron2}. In particular, the script \url{https://github.com/facebookresearch/detectron2/blob/master/tools/convert-torchvision-to-d2.py} allows us to convert a trained ResNet18 model from \texttt{torchvision} to the format needed for Detectron2. The repository has default configuration files for all experiments. We change the following fields to support using a frozen ResNet18:
\begin{lstlisting}
INPUT:
    FORMAT: RGB
MODEL:
    BACKBONE:
        FREEZE_AT: 5
    PIXEL_MEAN:
        - 123.675
        - 103.53
        - 116.28
    PIXEL_STD:
    - 58.395
    - 57.12
    - 57.375
    RESNETS:
        DEPTH: 18
        RES2_OUT_CHANNELS: 64
        STRIDE_IN_1X1: false
    WEIGHTS: <PATH_TO_CONVERTED_TORCHVISION_WEIGHTS>
\end{lstlisting}
We acknowledge that ResNet50 and larger are the commonly used backbones, so our results will not be  state-of-the-art. However, the ordering in performance between algorithms is still meaningful and our primary interest. Future work can explore larger architectures.

\section{Additional Experiments}
We discuss a few observations surrounding Ring Discrimination and in particular, annealing.

\textbf{Hard negative mining is not always productive.}
% In our experiments, we have shown that hard negative mining in the context of contrastive learning can result in better performing embeddings. However, we observe that it is not the case that choosing difficult negative examples is always good.
% For instance: early in training, when the learned embeddings have yet to converge to a sensible ordering on the hypersphere, using very difficult negatives can be detrimental, collapsing the model to a poor local minima.
We can attribute this to the poor quality of embeddings early in training: using hard negatives can (1) simply be too difficult for the encoder to discriminate, or (2) focus the embedding on smaller, perhaps spurious, differences between the instance and the hard negatives, rather than prioritizing higher level semantic information (e.g. object identity).
As a demonstration of this phenomena, Fig.~\ref{fig:discussion}a shows several training runs of IRing on CIFAR10 with varying thresholds $\omega_\ell$ initialized at every 50 epochs of an IR model for a total of 200 epochs. In the legend, a smaller percentage indicates drawing negatives more similar to the embedding of the current instance as measured by dot products. (IRing (100\%) and IR are identical.) The y-axis plots the accuracy of classification where for each test example, we predict the label of its L$_2$ nearest neighbor in the training split \citep{wu2018unsupervised,zhuang2019local}.
Fig.~\ref{fig:discussion} shows that (1) using smaller thresholds at the beginning of training results in lower test accuracies; (2) in the middle of training (epoch 50), the performance is equivalent for all models; and (3) in later training stages (epoch 100), using more difficult negatives is better. Notice the ordering of the lines in Fig.~\ref{fig:discussion}: $10\% < 25\% < 50\% < 100\%$  early in training while the inequalities are flipped at epoch 100.

\begin{figure}[h!]
  \centering
  \begin{subfigure}[b]{0.32\textwidth}
    \centering
    \includegraphics[width=\textwidth]{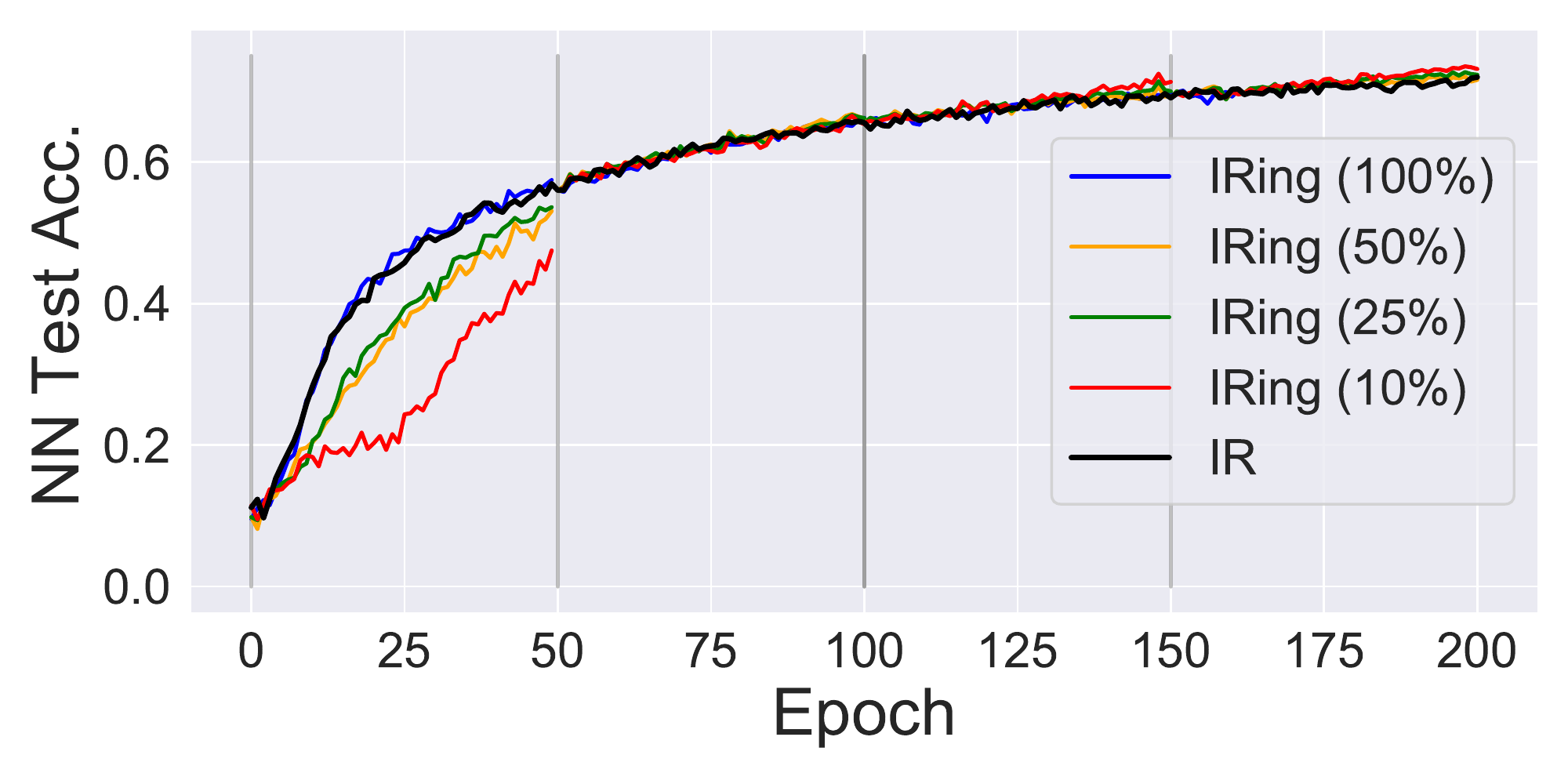}
    \caption{Effect of Hard Negatives}
  \end{subfigure}
  \begin{subfigure}[b]{0.32\textwidth}
    \centering
    \includegraphics[width=\textwidth]{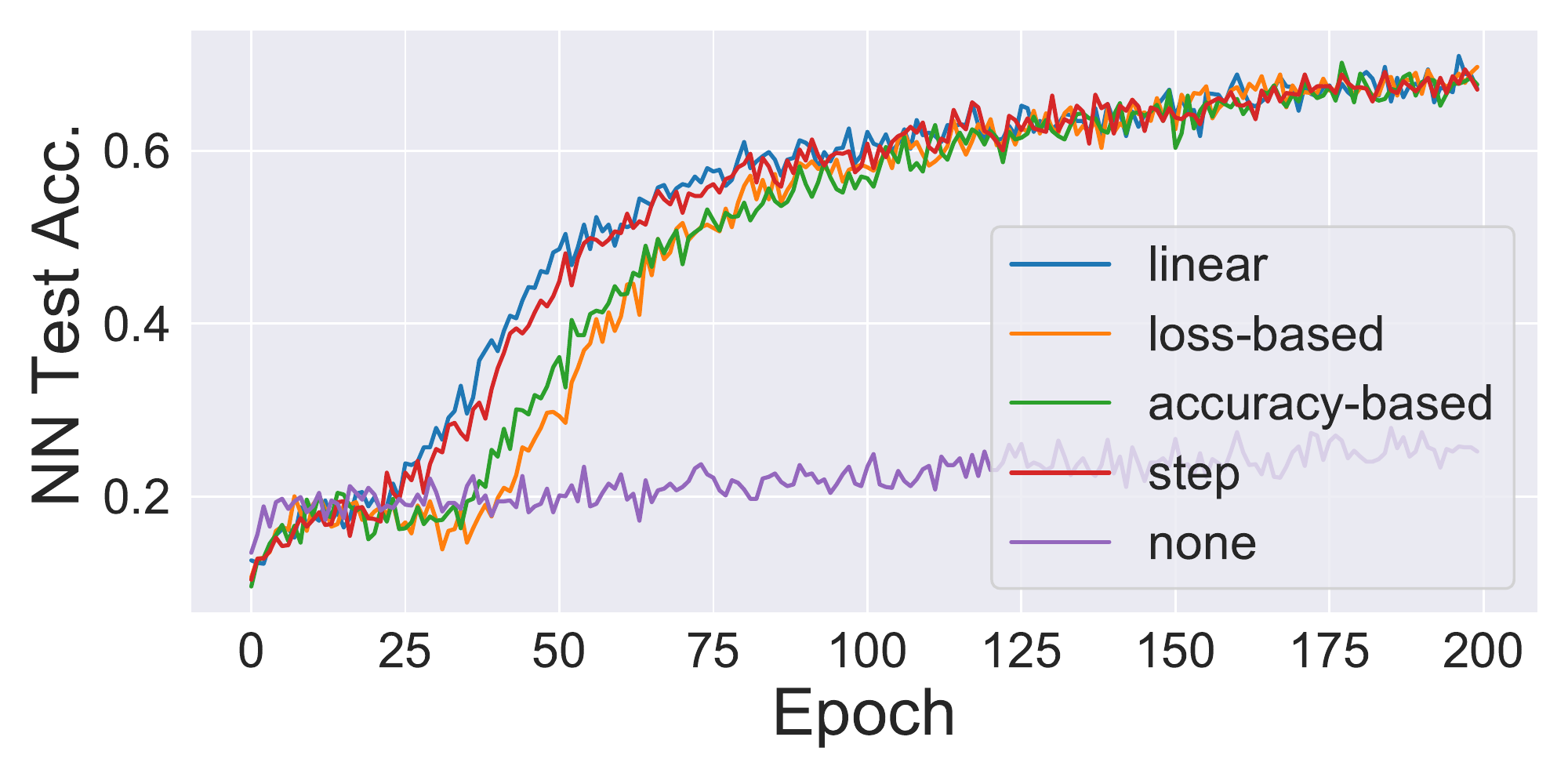}
    \caption{Annealing Policies}
  \end{subfigure}
  \begin{subfigure}[b]{0.32\textwidth}
    \centering
    \includegraphics[width=\textwidth]{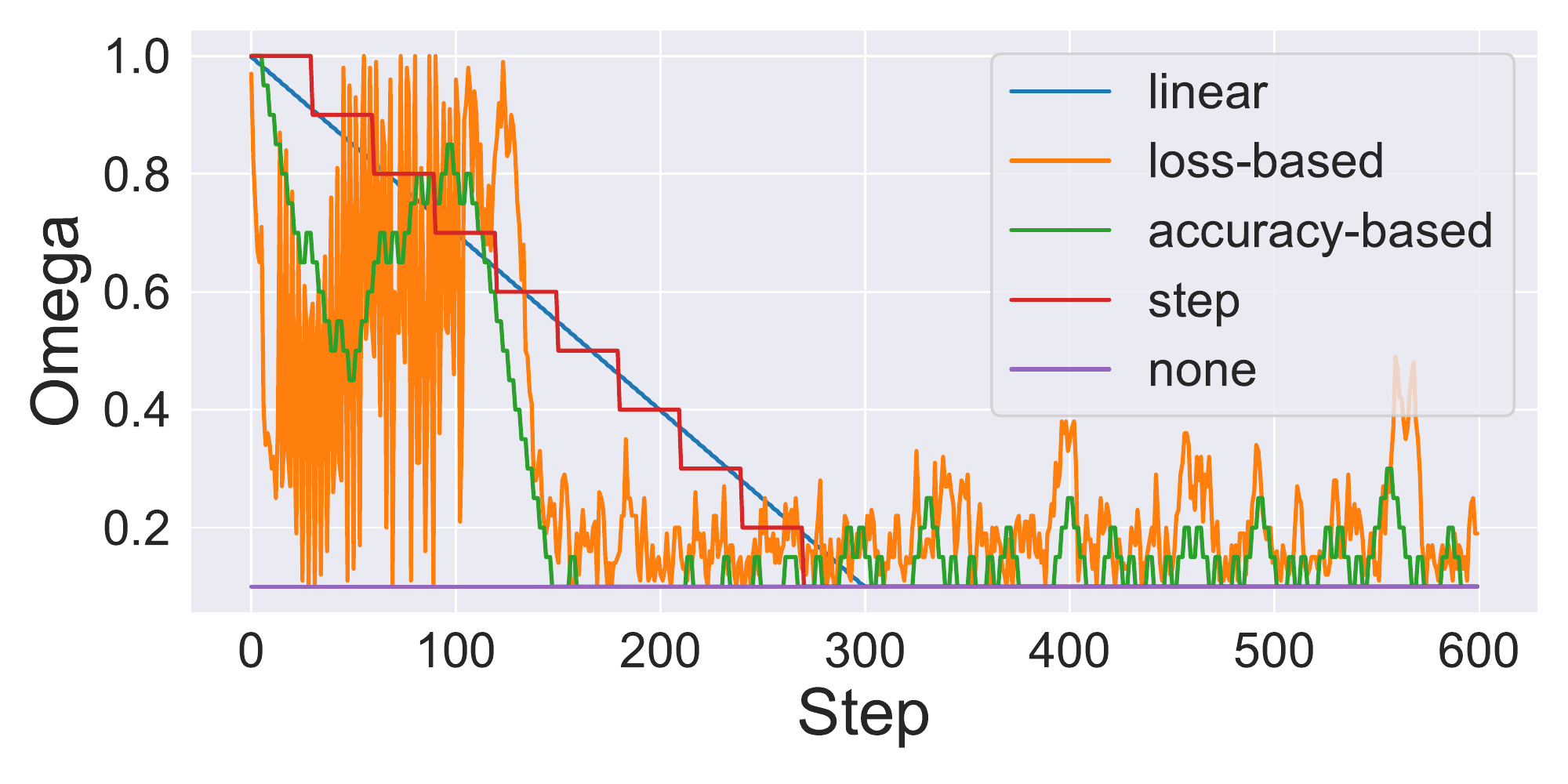}
    \caption{Annealing Thresholds $w_u$}
  \end{subfigure}
  \caption{(a) Embedding quality as a function of how similar negative samples are to the current instance in Ring Discrimination (the percentage represents the threshold $w_\ell$). (b,c) An exploration of difficult four policies for annealing Ring thresholds $w_\ell$.}
  \label{fig:discussion}
\end{figure}

\textbf{Exploring annealing policies.} Given that our experiments show annealing is important, there is a question of ``how to anneal''. In our experiments, we opted for a simple linear policy: slowly reducing $w_u$ from 100\% to 10\% in 100 epochs and maintaining it constant at 10\%  for the remaining epochs. Here, we briefly compare this to three other policies: a step function; an adaptive policy that lowers the threshold every epoch if the performance on a validation set increases, otherwise decreasing the threshold; and a similar adaptive policy that updates every step based on negative training loss. Fig.~\ref{fig:discussion}b compares the nearest neighbor test accuracies over 200 epochs of training IRing on CIFAR10 whereas Fig.~\ref{fig:discussion}b plots the threshold $w_u$.
We find that all the policies converge to roughly the same test accuracy, although linear and step policies appear to converge more quickly. From Fig.~\ref{fig:discussion}c, we observe that the adaptive methods naturally push the threshold down to 10\% (the lowest allowed threshold) around step 150, confirming our intuition that a smaller threshold later in training is desirable. Future work could explore more sophisticated policies.

\section{Related Work: Ring and Local Aggregation}
\label{sec:la}
Of the many algorithms listed above, we focus on Local Aggregation \citep{zhuang2019local}, or LA, which we conjecture to already be (implicitly) mining hard negatives.
While IR seeks to uniformly distribute embeddings, uniformity may not be optimal in all cases. For instance, images of the same class should intuitively be closer together than other images. The LA objective captures this intuition using a ``close neighbor set'' $C_i$ and ``background neighbor set'' $B_i$ conditioned on the current transformed instance $t(x_i)$. The background neighbor set contains the indices of elements in the dataset whose embeddings are closest to $g_\theta(t(x_i))$ in L$_2$ distance. The close neighbor set contains elements are same cluster as $t(x_i)$ using Kmeans assignments. Although not originally formulated in this manner, we can view the background neighbor set as being sampled from a variational distribution $q(B_i|t(x_i))$ with the lower threshold $w_\ell$ set to $0$ i.e. the ring is fully enclosed.  Now, writing LA in the notation of Eq.~\ref{eq:contrastive}, its objective is
\begin{equation}
  \mathcal{L}_{\text{LA}}(x_i; M) = \mathbf{E}_{t \sim p(t)} \mathbf{E}_{B_i \sim q(B_i|t(x_i))} \left[ \log \frac{\frac{1}{|C_i|}\sum_{j\in C_i} e^{g_\theta(t(x_i))^T M[j] / \tau} }{\frac{1}{|B_i|}\sum_{j' \in B_i} e^{g_\theta(t(x_i))^T M[j']/\tau}} \right].
  \label{eqn:la}
\end{equation}
Although Ring Discrimination and LA both mine hard negatives, LA additionally uses instances in the same KMeans cluster as positive views of $x_i$. Borrowing ideas from LA, we can explore several extensions of Ring Discrimination.
\begin{table}[h!]
\small
\centering
\begin{tabular}{l|c}
\toprule
Model & Top1 \\
\midrule
LA & 83.9 \\
IRCave & 84.0 \\
CMCCave & 87.2 \\
IRing (+$C_i$) & 84.3 \\
CMCRing (+$C_i$) & 87.8 \\
\bottomrule
\end{tabular}
\caption{Variants of Ring}
\label{table:la}
\end{table}
First, by ``Cave'' Discrimination (including IRCave ad CMCCave), we denote drawing negative samples from a CNCE distribution $q$ with a support restricted to the examples in the same KMeans clustering as the current instance (Note that such a definition falls under Theorem~\ref{thm:vince} as a particular choice for the restricted set $S_B$). Second, Ring (+$C_i$) instead, includes members of the KMeans clustering as positive views of $x_i$, like in LA --- here, negative samples are drawn as in regular Ring. Note that LA and IRing (+$C_i$) differ only by the lower threshold $w_\ell$, which is zero in the former and nonzero in the latter.
Table~\ref{table:la} shows promising results on CIFAR10 as these variations produce strong representations. This suggests that choosing good views and good negatives together can build even better contrastive algorithms.
\end{document}

%% file: math_commands.tex
\usepackage{amsmath,amsfonts,bm}

\def\eqref#1{equation~\ref{#1}}
\def\1{\bm{1}}

\DeclareMathAlphabet{\mathsfit}{\encodingdefault}{\sfdefault}{m}{sl}
\SetMathAlphabet{\mathsfit}{bold}{\encodingdefault}{\sfdefault}{bx}{n}

%% file: draft.bbl
\begin{thebibliography}{55}
\providecommand{\natexlab}[1]{#1}
\providecommand{\url}[1]{\texttt{#1}}
\expandafter\ifx\csname urlstyle\endcsname\relax
  \providecommand{\doi}[1]{doi: #1}\else
  \providecommand{\doi}{doi: \begingroup \urlstyle{rm}\Url}\fi

\bibitem[Asano et~al.(2019)Asano, Rupprecht, and Vedaldi]{asano2019self}
Yuki~Markus Asano, Christian Rupprecht, and Andrea Vedaldi.
\newblock Self-labelling via simultaneous clustering and representation
  learning.
\newblock \emph{arXiv preprint arXiv:1911.05371}, 2019.

\bibitem[Bachman et~al.(2019)Bachman, Hjelm, and
  Buchwalter]{bachman2019learning}
Philip Bachman, R~Devon Hjelm, and William Buchwalter.
\newblock Learning representations by maximizing mutual information across
  views.
\newblock In \emph{Advances in Neural Information Processing Systems}, pp.\
  15535--15545, 2019.

\bibitem[Becker \& Hinton(1992)Becker and Hinton]{becker1992self}
Suzanna Becker and Geoffrey~E Hinton.
\newblock Self-organizing neural network that discovers surfaces in random-dot
  stereograms.
\newblock \emph{Nature}, 355\penalty0 (6356):\penalty0 161--163, 1992.

\bibitem[Brown et~al.(2020)Brown, Mann, Ryder, Subbiah, Kaplan, Dhariwal,
  Neelakantan, Shyam, Sastry, Askell, et~al.]{brown2020language}
Tom~B Brown, Benjamin Mann, Nick Ryder, Melanie Subbiah, Jared Kaplan, Prafulla
  Dhariwal, Arvind Neelakantan, Pranav Shyam, Girish Sastry, Amanda Askell,
  et~al.
\newblock Language models are few-shot learners.
\newblock \emph{arXiv preprint arXiv:2005.14165}, 2020.

\bibitem[Caron et~al.(2020)Caron, Misra, Mairal, Goyal, Bojanowski, and
  Joulin]{caron2020unsupervised}
Mathilde Caron, Ishan Misra, Julien Mairal, Priya Goyal, Piotr Bojanowski, and
  Armand Joulin.
\newblock Unsupervised learning of visual features by contrasting cluster
  assignments.
\newblock \emph{arXiv preprint arXiv:2006.09882}, 2020.

\bibitem[Chen et~al.(2019)Chen, Zhai, Ritter, Lucic, and Houlsby]{chen2019self}
Ting Chen, Xiaohua Zhai, Marvin Ritter, Mario Lucic, and Neil Houlsby.
\newblock Self-supervised gans via auxiliary rotation loss.
\newblock In \emph{Proceedings of the IEEE Conference on Computer Vision and
  Pattern Recognition}, pp.\  12154--12163, 2019.

\bibitem[Chen et~al.(2020{\natexlab{a}})Chen, Kornblith, Norouzi, and
  Hinton]{chen2020simple}
Ting Chen, Simon Kornblith, Mohammad Norouzi, and Geoffrey Hinton.
\newblock A simple framework for contrastive learning of visual
  representations.
\newblock \emph{arXiv preprint arXiv:2002.05709}, 2020{\natexlab{a}}.

\bibitem[Chen et~al.(2020{\natexlab{b}})Chen, Fan, Girshick, and
  He]{chen2020improved}
Xinlei Chen, Haoqi Fan, Ross Girshick, and Kaiming He.
\newblock Improved baselines with momentum contrastive learning.
\newblock \emph{arXiv preprint arXiv:2003.04297}, 2020{\natexlab{b}}.

\bibitem[Deng et~al.(2009)Deng, Dong, Socher, Li, Li, and
  Fei-Fei]{deng2009imagenet}
Jia Deng, Wei Dong, Richard Socher, Li-Jia Li, Kai Li, and Li~Fei-Fei.
\newblock Imagenet: A large-scale hierarchical image database.
\newblock In \emph{2009 IEEE conference on computer vision and pattern
  recognition}, pp.\  248--255. Ieee, 2009.

\bibitem[Devlin et~al.(2018)Devlin, Chang, Lee, and Toutanova]{devlin2018bert}
Jacob Devlin, Ming-Wei Chang, Kenton Lee, and Kristina Toutanova.
\newblock Bert: Pre-training of deep bidirectional transformers for language
  understanding.
\newblock \emph{arXiv preprint arXiv:1810.04805}, 2018.

\bibitem[Doersch \& Zisserman(2017)Doersch and Zisserman]{doersch2017multi}
Carl Doersch and Andrew Zisserman.
\newblock Multi-task self-supervised visual learning.
\newblock In \emph{Proceedings of the IEEE International Conference on Computer
  Vision}, pp.\  2051--2060, 2017.

\bibitem[Doersch et~al.(2015)Doersch, Gupta, and
  Efros]{doersch2015unsupervised}
Carl Doersch, Abhinav Gupta, and Alexei~A Efros.
\newblock Unsupervised visual representation learning by context prediction.
\newblock In \emph{Proceedings of the IEEE International Conference on Computer
  Vision}, pp.\  1422--1430, 2015.

\bibitem[Donahue \& Simonyan(2019)Donahue and Simonyan]{donahue2019large}
Jeff Donahue and Karen Simonyan.
\newblock Large scale adversarial representation learning.
\newblock In \emph{Advances in Neural Information Processing Systems}, pp.\
  10542--10552, 2019.

\bibitem[Dosovitskiy et~al.(2014)Dosovitskiy, Springenberg, Riedmiller, and
  Brox]{dosovitskiy2014discriminative}
Alexey Dosovitskiy, Jost~Tobias Springenberg, Martin Riedmiller, and Thomas
  Brox.
\newblock Discriminative unsupervised feature learning with convolutional
  neural networks.
\newblock In \emph{Advances in neural information processing systems}, pp.\
  766--774, 2014.

\bibitem[Everingham et~al.(2010)Everingham, Van~Gool, Williams, Winn, and
  Zisserman]{everingham2010pascal}
Mark Everingham, Luc Van~Gool, Christopher~KI Williams, John Winn, and Andrew
  Zisserman.
\newblock The pascal visual object classes (voc) challenge.
\newblock \emph{International journal of computer vision}, 88\penalty0
  (2):\penalty0 303--338, 2010.

\bibitem[Gidaris et~al.(2018)Gidaris, Singh, and
  Komodakis]{gidaris2018unsupervised}
Spyros Gidaris, Praveer Singh, and Nikos Komodakis.
\newblock Unsupervised representation learning by predicting image rotations.
\newblock \emph{arXiv preprint arXiv:1803.07728}, 2018.

\bibitem[Grill et~al.(2020)Grill, Strub, Altch{\'e}, Tallec, Richemond,
  Buchatskaya, Doersch, Pires, Guo, Azar, et~al.]{grill2020bootstrap}
Jean-Bastien Grill, Florian Strub, Florent Altch{\'e}, Corentin Tallec,
  Pierre~H Richemond, Elena Buchatskaya, Carl Doersch, Bernardo~Avila Pires,
  Zhaohan~Daniel Guo, Mohammad~Gheshlaghi Azar, et~al.
\newblock Bootstrap your own latent: A new approach to self-supervised
  learning.
\newblock \emph{arXiv preprint arXiv:2006.07733}, 2020.

\bibitem[Gutmann \& Hyv{\"a}rinen(2010)Gutmann and
  Hyv{\"a}rinen]{gutmann2010noise}
Michael Gutmann and Aapo Hyv{\"a}rinen.
\newblock Noise-contrastive estimation: A new estimation principle for
  unnormalized statistical models.
\newblock In \emph{Proceedings of the Thirteenth International Conference on
  Artificial Intelligence and Statistics}, pp.\  297--304, 2010.

\bibitem[Hadsell et~al.(2006)Hadsell, Chopra, and
  LeCun]{hadsell2006dimensionality}
Raia Hadsell, Sumit Chopra, and Yann LeCun.
\newblock Dimensionality reduction by learning an invariant mapping.
\newblock In \emph{2006 IEEE Computer Society Conference on Computer Vision and
  Pattern Recognition (CVPR'06)}, volume~2, pp.\  1735--1742. IEEE, 2006.

\bibitem[He et~al.(2016)He, Zhang, Ren, and Sun]{he2016deep}
Kaiming He, Xiangyu Zhang, Shaoqing Ren, and Jian Sun.
\newblock Deep residual learning for image recognition.
\newblock In \emph{Proceedings of the IEEE conference on computer vision and
  pattern recognition}, pp.\  770--778, 2016.

\bibitem[He et~al.(2017)He, Gkioxari, Doll{\'a}r, and Girshick]{he2017mask}
Kaiming He, Georgia Gkioxari, Piotr Doll{\'a}r, and Ross Girshick.
\newblock Mask r-cnn.
\newblock In \emph{Proceedings of the IEEE international conference on computer
  vision}, pp.\  2961--2969, 2017.

\bibitem[He et~al.(2019)He, Fan, Wu, Xie, and Girshick]{he2019momentum}
Kaiming He, Haoqi Fan, Yuxin Wu, Saining Xie, and Ross Girshick.
\newblock Momentum contrast for unsupervised visual representation learning.
\newblock \emph{arXiv preprint arXiv:1911.05722}, 2019.

\bibitem[H{\'e}naff et~al.(2019)H{\'e}naff, Srinivas, De~Fauw, Razavi, Doersch,
  Eslami, and Oord]{henaff2019data}
Olivier~J H{\'e}naff, Aravind Srinivas, Jeffrey De~Fauw, Ali Razavi, Carl
  Doersch, SM~Eslami, and Aaron van~den Oord.
\newblock Data-efficient image recognition with contrastive predictive coding.
\newblock \emph{arXiv preprint arXiv:1905.09272}, 2019.

\bibitem[Hjelm et~al.(2018)Hjelm, Fedorov, Lavoie-Marchildon, Grewal, Bachman,
  Trischler, and Bengio]{hjelm2018learning}
R~Devon Hjelm, Alex Fedorov, Samuel Lavoie-Marchildon, Karan Grewal, Phil
  Bachman, Adam Trischler, and Yoshua Bengio.
\newblock Learning deep representations by mutual information estimation and
  maximization.
\newblock \emph{arXiv preprint arXiv:1808.06670}, 2018.

\bibitem[Ioffe \& Szegedy(2015)Ioffe and Szegedy]{ioffe2015batch}
Sergey Ioffe and Christian Szegedy.
\newblock Batch normalization: Accelerating deep network training by reducing
  internal covariate shift.
\newblock \emph{arXiv preprint arXiv:1502.03167}, 2015.

\bibitem[Iscen et~al.(2019)Iscen, Tolias, Avrithis, and Chum]{iscen2019label}
Ahmet Iscen, Giorgos Tolias, Yannis Avrithis, and Ondrej Chum.
\newblock Label propagation for deep semi-supervised learning.
\newblock In \emph{Proceedings of the IEEE conference on computer vision and
  pattern recognition}, pp.\  5070--5079, 2019.

\bibitem[Ji et~al.(2019)Ji, Henriques, and Vedaldi]{ji2019invariant}
Xu~Ji, Jo{\~a}o~F Henriques, and Andrea Vedaldi.
\newblock Invariant information clustering for unsupervised image
  classification and segmentation.
\newblock In \emph{Proceedings of the IEEE International Conference on Computer
  Vision}, pp.\  9865--9874, 2019.

\bibitem[Kolesnikov et~al.(2019)Kolesnikov, Zhai, and
  Beyer]{kolesnikov2019revisiting}
Alexander Kolesnikov, Xiaohua Zhai, and Lucas Beyer.
\newblock Revisiting self-supervised visual representation learning.
\newblock In \emph{Proceedings of the IEEE conference on Computer Vision and
  Pattern Recognition}, pp.\  1920--1929, 2019.

\bibitem[Li et~al.(2020)Li, Zhou, Xiong, Socher, and Hoi]{li2020prototypical}
Junnan Li, Pan Zhou, Caiming Xiong, Richard Socher, and Steven~CH Hoi.
\newblock Prototypical contrastive learning of unsupervised representations.
\newblock \emph{arXiv preprint arXiv:2005.04966}, 2020.

\bibitem[Lin et~al.(2014)Lin, Maire, Belongie, Hays, Perona, Ramanan,
  Doll{\'a}r, and Zitnick]{lin2014microsoft}
Tsung-Yi Lin, Michael Maire, Serge Belongie, James Hays, Pietro Perona, Deva
  Ramanan, Piotr Doll{\'a}r, and C~Lawrence Zitnick.
\newblock Microsoft coco: Common objects in context.
\newblock In \emph{European conference on computer vision}, pp.\  740--755.
  Springer, 2014.

\bibitem[Liu et~al.(2019)Liu, Ott, Goyal, Du, Joshi, Chen, Levy, Lewis,
  Zettlemoyer, and Stoyanov]{liu2019roberta}
Yinhan Liu, Myle Ott, Naman Goyal, Jingfei Du, Mandar Joshi, Danqi Chen, Omer
  Levy, Mike Lewis, Luke Zettlemoyer, and Veselin Stoyanov.
\newblock Roberta: A robustly optimized bert pretraining approach.
\newblock \emph{arXiv preprint arXiv:1907.11692}, 2019.

\bibitem[Misra \& Maaten(2020)Misra and Maaten]{misra2020self}
Ishan Misra and Laurens van~der Maaten.
\newblock Self-supervised learning of pretext-invariant representations.
\newblock In \emph{Proceedings of the IEEE/CVF Conference on Computer Vision
  and Pattern Recognition}, pp.\  6707--6717, 2020.

\bibitem[Oh~Song et~al.(2016)Oh~Song, Xiang, Jegelka, and Savarese]{oh2016deep}
Hyun Oh~Song, Yu~Xiang, Stefanie Jegelka, and Silvio Savarese.
\newblock Deep metric learning via lifted structured feature embedding.
\newblock In \emph{Proceedings of the IEEE conference on computer vision and
  pattern recognition}, pp.\  4004--4012, 2016.

\bibitem[Oord et~al.(2018)Oord, Li, and Vinyals]{oord2018representation}
Aaron van~den Oord, Yazhe Li, and Oriol Vinyals.
\newblock Representation learning with contrastive predictive coding.
\newblock \emph{arXiv preprint arXiv:1807.03748}, 2018.

\bibitem[Pathak et~al.(2017)Pathak, Girshick, Doll{\'a}r, Darrell, and
  Hariharan]{pathak2017learning}
Deepak Pathak, Ross Girshick, Piotr Doll{\'a}r, Trevor Darrell, and Bharath
  Hariharan.
\newblock Learning features by watching objects move.
\newblock In \emph{Proceedings of the IEEE Conference on Computer Vision and
  Pattern Recognition}, pp.\  2701--2710, 2017.

\bibitem[Poole et~al.(2019)Poole, Ozair, Oord, Alemi, and
  Tucker]{poole2019variational}
Ben Poole, Sherjil Ozair, Aaron van~den Oord, Alexander~A Alemi, and George
  Tucker.
\newblock On variational bounds of mutual information.
\newblock \emph{arXiv preprint arXiv:1905.06922}, 2019.

\bibitem[Radford et~al.(2018)Radford, Narasimhan, Salimans, and
  Sutskever]{radford2018improving}
Alec Radford, Karthik Narasimhan, Tim Salimans, and Ilya Sutskever.
\newblock Improving language understanding by generative pre-training, 2018.

\bibitem[Radford et~al.(2019)Radford, Wu, Child, Luan, Amodei, and
  Sutskever]{radford2019language}
Alec Radford, Jeffrey Wu, Rewon Child, David Luan, Dario Amodei, and Ilya
  Sutskever.
\newblock Language models are unsupervised multitask learners.
\newblock \emph{OpenAI Blog}, 1\penalty0 (8):\penalty0 9, 2019.

\bibitem[Redmon et~al.(2016)Redmon, Divvala, Girshick, and
  Farhadi]{redmon2016you}
Joseph Redmon, Santosh Divvala, Ross Girshick, and Ali Farhadi.
\newblock You only look once: Unified, real-time object detection.
\newblock In \emph{Proceedings of the IEEE conference on computer vision and
  pattern recognition}, pp.\  779--788, 2016.

\bibitem[Russakovsky et~al.(2015)Russakovsky, Deng, Su, Krause, Satheesh, Ma,
  Huang, Karpathy, Khosla, Bernstein, et~al.]{russakovsky2015imagenet}
Olga Russakovsky, Jia Deng, Hao Su, Jonathan Krause, Sanjeev Satheesh, Sean Ma,
  Zhiheng Huang, Andrej Karpathy, Aditya Khosla, Michael Bernstein, et~al.
\newblock Imagenet large scale visual recognition challenge.
\newblock \emph{International journal of computer vision}, 115\penalty0
  (3):\penalty0 211--252, 2015.

\bibitem[Schroff et~al.(2015)Schroff, Kalenichenko, and
  Philbin]{schroff2015facenet}
Florian Schroff, Dmitry Kalenichenko, and James Philbin.
\newblock Facenet: A unified embedding for face recognition and clustering.
\newblock In \emph{Proceedings of the IEEE conference on computer vision and
  pattern recognition}, pp.\  815--823, 2015.

\bibitem[Srinivas et~al.(2020)Srinivas, Laskin, and Abbeel]{srinivas2020curl}
Aravind Srinivas, Michael Laskin, and Pieter Abbeel.
\newblock Curl: Contrastive unsupervised representations for reinforcement
  learning.
\newblock \emph{arXiv preprint arXiv:2004.04136}, 2020.

\bibitem[Tian et~al.(2019)Tian, Krishnan, and Isola]{tian2019contrastive}
Yonglong Tian, Dilip Krishnan, and Phillip Isola.
\newblock Contrastive multiview coding.
\newblock \emph{arXiv preprint arXiv:1906.05849}, 2019.

\bibitem[Tian et~al.(2020)Tian, Sun, Poole, Krishnan, Schmid, and
  Isola]{tian2020makes}
Yonglong Tian, Chen Sun, Ben Poole, Dilip Krishnan, Cordelia Schmid, and
  Phillip Isola.
\newblock What makes for good views for contrastive learning.
\newblock \emph{arXiv preprint arXiv:2005.10243}, 2020.

\bibitem[Triantafillou et~al.(2019)Triantafillou, Zhu, Dumoulin, Lamblin, Evci,
  Xu, Goroshin, Gelada, Swersky, Manzagol, et~al.]{triantafillou2019meta}
Eleni Triantafillou, Tyler Zhu, Vincent Dumoulin, Pascal Lamblin, Utku Evci,
  Kelvin Xu, Ross Goroshin, Carles Gelada, Kevin Swersky, Pierre-Antoine
  Manzagol, et~al.
\newblock Meta-dataset: A dataset of datasets for learning to learn from few
  examples.
\newblock \emph{arXiv preprint arXiv:1903.03096}, 2019.

\bibitem[Tschannen et~al.(2019)Tschannen, Djolonga, Rubenstein, Gelly, and
  Lucic]{tschannen2019mutual}
Michael Tschannen, Josip Djolonga, Paul~K Rubenstein, Sylvain Gelly, and Mario
  Lucic.
\newblock On mutual information maximization for representation learning.
\newblock \emph{arXiv preprint arXiv:1907.13625}, 2019.

\bibitem[Wang \& Gupta(2015)Wang and Gupta]{wang2015unsupervised}
Xiaolong Wang and Abhinav Gupta.
\newblock Unsupervised learning of visual representations using videos.
\newblock In \emph{Proceedings of the IEEE international conference on computer
  vision}, pp.\  2794--2802, 2015.

\bibitem[Wu et~al.(2017)Wu, Manmatha, Smola, and Krahenbuhl]{wu2017sampling}
Chao-Yuan Wu, R~Manmatha, Alexander~J Smola, and Philipp Krahenbuhl.
\newblock Sampling matters in deep embedding learning.
\newblock In \emph{Proceedings of the IEEE International Conference on Computer
  Vision}, pp.\  2840--2848, 2017.

\bibitem[Wu et~al.(2019)Wu, Kirillov, Massa, Lo, and
  Girshick]{wu2019detectron2}
Yuxin Wu, Alexander Kirillov, Francisco Massa, Wan-Yen Lo, and Ross Girshick.
\newblock Detectron2.
\newblock \url{https://github.com/facebookresearch/detectron2}, 2019.

\bibitem[Wu et~al.(2018)Wu, Xiong, Yu, and Lin]{wu2018unsupervised}
Zhirong Wu, Yuanjun Xiong, Stella~X Yu, and Dahua Lin.
\newblock Unsupervised feature learning via non-parametric instance
  discrimination.
\newblock In \emph{Proceedings of the IEEE Conference on Computer Vision and
  Pattern Recognition}, pp.\  3733--3742, 2018.

\bibitem[Xie et~al.(2020)Xie, Zhan, Liu, Ong, and Loy]{xie2020delving}
Jiahao Xie, Xiaohang Zhan, Ziwei Liu, Yew~Soon Ong, and Chen~Change Loy.
\newblock Delving into inter-image invariance for unsupervised visual
  representations.
\newblock \emph{arXiv preprint arXiv:2008.11702}, 2020.

\bibitem[Ye et~al.(2019)Ye, Zhang, Yuen, and Chang]{ye2019unsupervised}
Mang Ye, Xu~Zhang, Pong~C Yuen, and Shih-Fu Chang.
\newblock Unsupervised embedding learning via invariant and spreading instance
  feature.
\newblock In \emph{Proceedings of the IEEE Conference on computer vision and
  pattern recognition}, pp.\  6210--6219, 2019.

\bibitem[Yuan et~al.(2017)Yuan, Yang, and Zhang]{yuan2017hard}
Yuhui Yuan, Kuiyuan Yang, and Chao Zhang.
\newblock Hard-aware deeply cascaded embedding.
\newblock In \emph{Proceedings of the IEEE international conference on computer
  vision}, pp.\  814--823, 2017.

\bibitem[Zhang et~al.(2016)Zhang, Isola, and Efros]{zhang2016colorful}
Richard Zhang, Phillip Isola, and Alexei~A Efros.
\newblock Colorful image colorization.
\newblock In \emph{European conference on computer vision}, pp.\  649--666.
  Springer, 2016.

\bibitem[Zhuang et~al.(2019)Zhuang, Zhai, and Yamins]{zhuang2019local}
Chengxu Zhuang, Alex~Lin Zhai, and Daniel Yamins.
\newblock Local aggregation for unsupervised learning of visual embeddings.
\newblock In \emph{Proceedings of the IEEE International Conference on Computer
  Vision}, pp.\  6002--6012, 2019.

\end{thebibliography}


\begin{thebibliography}{3}
\providecommand{\natexlab}[1]{#1}
\providecommand{\url}[1]{\texttt{#1}}
\expandafter\ifx\csname urlstyle\endcsname\relax
  \providecommand{\doi}[1]{doi: #1}\else
  \providecommand{\doi}{doi: \begingroup \urlstyle{rm}\Url}\fi

\bibitem[Bengio \& LeCun(2007)Bengio and LeCun]{Bengio+chapter2007}
Yoshua Bengio and Yann LeCun.
\newblock Scaling learning algorithms towards {AI}.
\newblock In \emph{Large Scale Kernel Machines}. MIT Press, 2007.

\bibitem[Goodfellow et~al.(2016)Goodfellow, Bengio, Courville, and
  Bengio]{goodfellow2016deep}
Ian Goodfellow, Yoshua Bengio, Aaron Courville, and Yoshua Bengio.
\newblock \emph{Deep learning}, volume~1.
\newblock MIT Press, 2016.

\bibitem[Hinton et~al.(2006)Hinton, Osindero, and Teh]{Hinton06}
Geoffrey~E. Hinton, Simon Osindero, and Yee~Whye Teh.
\newblock A fast learning algorithm for deep belief nets.
\newblock \emph{Neural Computation}, 18:\penalty0 1527--1554, 2006.

\end{thebibliography}
